\documentclass[11pt, letterpaper]{article}

\usepackage[utf8]{inputenc} % allow utf-8 input
\usepackage{hyperref}       % hyperlinks
\usepackage{url}            % simple URL typesetting
\usepackage{booktabs}       % professional-quality tables
\usepackage{amsfonts}       % blackboard math symbols
\usepackage{nicefrac}       % compact symbols for 1/2, etc.
\usepackage{microtype}      % microtypography
\usepackage{xcolor}         % colors
\RequirePackage{amsthm,amsmath,amsfonts,amssymb}

\usepackage[capitalise]{cleveref}
\usepackage[top=1in, bottom=1in, left=1in, right=1in]{geometry}
\usepackage{float}
\usepackage{subfig}
\usepackage{epsfig}
\usepackage[normalem]{ulem}
\usepackage{multicol}
\usepackage{sectsty}
\usepackage{tikz}
\usetikzlibrary{shapes,arrows,positioning}

\providecommand{\tauDK}{\hat {\rm ATE}_{\rm DQ}}

\newcommand{\field}[1]{\ensuremath{\mathbb{#1}}}
\newcommand{\R}{\ensuremath{\field{R}}} % real numbers
\newcommand{\Z}{\ensuremath{\field{Z}}} % integers
\newcommand{\I}[1]{\ensuremath{\mathbb{I}_{\left\{#1\right\}}}} % indicator function
\newcommand{\E}{\ensuremath{\mathsf{E}}} % expectation

\newcommand{\Ascr}{\ensuremath{\mathcal A}}

\newcommand{\Sscr}{\ensuremath{\mathcal S}}

\providecommand{\norm[1]}{\lVert #1 \rVert}

\providecommand{\sigon}{\sigma_{\mathrm{on}}}
\providecommand{\one}{\mathbf{1}}
\providecommand{\eoff}{\tau_{\mathrm{off}}}
\providecommand{\sigoff}{\sigma_{\mathrm{off}}}

\newtheoremstyle{thm-sf}{}{}{\itshape}{}{\sffamily\bfseries}{.}{ }{}
\theoremstyle{thm-sf}
\makeatletter
\newtheorem*{rep@theorem}{\rep@title}
\newcommand{\newreptheorem}[2]{%
  \newenvironment{rep#1}[1]{%
    \def\rep@title{#2 \ref{##1}}%
    \begin{rep@theorem}}%
    {\end{rep@theorem}}}
\makeatother
\newreptheorem{theorem}{Theorem}
\newreptheorem{lemma}{Lemma}

\newtheorem{assumption}{Assumption}

\newtheorem{theorem}{Theorem}

\newtheorem{lemma}{Lemma}

\usepackage{setspace}
\allsectionsfont{\sffamily}
\floatname{algorithm}{\normalfont\sffamily\bfseries Algorithm}

\providecommand{\norm}[1]{\left\lVert\mspace{1mu} #1 \mspace{1mu}\right\rVert}
\providecommand{\R}{\mathbb{R}}

\usepackage{titlesec}

\titlespacing{\paragraph}{0pt}{5pt}{10pt}
\titlespacing{\section}{0pt}{5pt}{5pt}
\titlespacing{\subsection}{0pt}{5pt}{5pt}

\title{\textsf{\textbf{Markovian Interference in Experiments}}}

\author{
  Vivek F. Farias \\
  Operations Research Center, Massachusetts Institute of Technology 
  \and
  Andrew A. Li \\
  Tepper School of Business, Carnegie Mellon University 
  \and
  Tianyi Peng \\
  Department of Aeronautics and Astronautics, Massachusetts Institute of Technology 
  \and
  Andrew Zheng \\
  Operations Research Center, Massachusetts Institute of Technology
  }

\date{}

\begin{document}
\maketitle

%\vspace{-55pt}

\begin{abstract}
\noindent We consider experiments in dynamical systems where interventions on some experimental units impact other units through a limiting constraint (such as a limited inventory). Despite outsize practical importance, the best estimators for this `Markovian' interference problem are largely heuristic in nature, and their bias is not well understood. We formalize the problem of inference in such experiments as one of policy evaluation. Off-policy estimators, while unbiased, apparently incur a large penalty in variance relative to state-of-the-art heuristics. We introduce an on-policy estimator: the Differences-In-Q's (DQ) estimator. We show that the DQ estimator can in general have exponentially smaller variance than off-policy evaluation. At the same time, its bias is second order in the impact of the intervention. This yields a striking bias-variance tradeoff so that the DQ estimator effectively dominates state-of-the-art alternatives. From a theoretical perspective, we introduce three separate novel techniques that are of independent interest in the theory of Reinforcement Learning (RL). Our empirical evaluation includes a set of experiments on a city-scale ride-hailing simulator.   
\end{abstract}

% \newpage
%\setstretch{1.6}

\section{Introduction}

Experimentation is a broadly-deployed learning tool in online commerce that is simple to execute, in principle: apply the treatment in question at random (e.g.~an A/B test), and `naively' infer the average effect of the treatment by differencing the average outcomes under treatment and control. About a decade ago, Blake and Coey \cite{blake2014marketplace} pointed out a challenge in such experimentation on Ebay:

\vspace{1em}

\noindent {\em ``Consider the example of testing a new search engine ranking algorithm which steers test buyers towards a particular class of items for sale. If test users buy up those items, the supply available to the control users declines.''}  

\vspace{1em}

\noindent This violation of the so-called Stable Unit Treatment Value Assumption (SUTVA) \cite{cox1958planning} has been viewed as problematic in online platforms as early as Reiley's seminal `Magic on the Internet' work \cite{lucking-reileyUsingFieldExperiments1999}. Blake and Coey \cite{blake2014marketplace} were simply pointing out that the resulting inferential biases were large, which is particularly problematic since treatment effects in this context are typically tiny. The {\em interference} problem above is germane to experimentation on commerce platforms where interventions on a given experimental unit impact other units, since all units share a common inventory of `demand'  or `supply' depending on context.

Despite the ubiquity of such interference, a practical solution is far from settled. An ongoing line of work addresses the problem via {\em experimental design}, assigning treatments carefully to mitigate the bias of `naively'-derived estimators. In the best cases such designs provably reduce bias by exploiting certain application specific structures, but often it is unclear whether the problem at hand affords such structure (a case in point being the search-engine example above, as will be apparent later). As such, experimentation on online platforms still largely relies on simple randomization, i.e.~A/B tests. Motivated by this fact, we focus instead on {\em designing effective  estimators} assuming simple randomization. We demonstrate a novel estimator which, thanks to an effective bias-variance tradeoff, is a compelling alternative to both alternative state-of-the-art estimators as well as bespoke experimental designs when they apply. 

%\vspace{-1em}
\paragraph{Markovian Interference and Existing Approaches:} We study a generic experimentation problem within a system represented as a Markov Decision Process (MDP), where treatment corresponds to an action which may interfere with state transitions. This form of interference, which we refer to as {\em Markovian}, naturally subsumes the platform examples above, as recently noted by others either implicitly \cite{qinReinforcementLearningRidesharing2021} or explicitly \cite{johari2022experimental,shi2020reinforcement}. In that example, a user arrives at each time step, the platform chooses an action (whether to treat the user), and the user's purchase decision alters the system state (inventory levels).

Our goal is to estimate the Average Treatment Effect (ATE), defined as the difference in steady-state reward with and without applying the treatment. In light of the above discussion, we assume that experimentation is done under simple randomization (i.e.~A/B testing). Now without design as a lever, there are perhaps two existing families of estimators:
\newline
\textbf{1. Naive:} We will explicitly define the {\em Naive} estimator in the next section, but the strategy amounts to simply ignoring the presence of interference. This is by and large what is done in practice. Of course it may suffer from high bias (we show this momentarily in Section~\ref{sec:example}), but it serves as more than just a strawman. In particular, bias is only one side of the estimation coin, and with respect to the other side, namely variance, the Naive estimator is effectively the best possible.
\newline
\textbf{2. Off-Policy Evaluation (OPE):} Another approach comes from viewing our problem as one of policy evaluation in reinforcement learning (RL). Succinctly, it can be viewed as estimating the average reward of two different policies (no treatment, or treatment) given observations from some {\em third} policy (simple randomization). This immediately suggests framing the problem as one of {\em Off-Policy Evaluation}, and borrowing one of many existing {\em unbiased} estimators, e.g.~\cite{thomas2015high,thomas2016data,liu2018breaking,jiang2016doubly,kallus2020double,kallus2022efficiently}. This tack appears to be promising, e.g.~\cite{shi2020reinforcement}, but we observe that the resulting variance is necessarily large ({\bf Theorem \ref{th:cramer-rao}}).
%State-of the-art estimators for this problem proposed in recent years attempt to heuristically correct for the interference between experimental units and thus mitigate bias. But the extent to which these estimators mitigate the problem at hand is largely only understood experimentally. The problem at hand is a hard one since real-world experiments with such interference often have limited experimentation budgets, which ends up ruling out more sophisticated estimators due to their variance.

%\vspace{-1em}
\paragraph{Our Contributions:}
%This paper seeks to construct a practical solution to the problem of inference in experiments with Markovian interference. We do so by viewing the problem at hand as one of policy evaluation in reinforcement learning (RL). Succinctly, the mathematical task implicit in our problem can be viewed as estimating the difference in the average reward of a {Markov Decision Process} (MDP) under one of two policies corresponding to whether or not the intervention is implemented. This difference is in essence the so-called Average Treatment Effect (ATE) of the intervention. The challenge is that we must estimate the ATE given observations from some {\em third} policy corresponding to the experimental design employed. This immediately suggests using ideas from Off-Policy Evaluation to construct unbiased estimators of the ATE. Disappointingly, we observe that while such estimators are unbiased, their variance is substantially larger than incumbent state-of-the art heuristics.
Against the above backdrop, we propose a novel {\em on}-policy treatment-effect estimator, which we dub the `Differences-In-Q's' (DQ) estimator, %\footnote{The name is a nod to a theorem of \cite{dynkin1965markov} that is oft-used (though rarely cited) in the control literature.}
for experiments with Markovian interference. In a nutshell, we characterize our contribution as follows:

\noindent {\em The DQ estimator has provably negligible bias relative to the treatment effect. Its variance can, in general be exponentially smaller than that of an efficent off-policy estimator. In both stylized and large-scale real-world models, it dominates state-of-the-art alternatives}.

We next describe these relative merits in greater detail:
\newline
\textbf{1. Second-order Bias: }We show ({\bf Theorem \ref{th:dynkin_bias}}) that when the impact of an intervention on transition probabilities is $O(\delta)$, the bias of the DQ estimator is $O(\delta^2)$. The DQ estimator thus leverages the one piece of structure we have relative to generic off-policy evaluation: treatment effects are typically small. Our analysis introduces a novel Taylor-like expansion of the ATE (Theorem~\ref{th:korder-bias}) that in addition to the current setting, is of general interest in the theory of RL (for instance, in the context of Policy Optimization). 
% wherein the suppressed constants depend {\em logarithmically} on the size of the state space.
%One approach to dealing with the challenge above is collecting a sufficient amount of data under the design policy so as to effectively reconstruct the steady state distribution under both the intervention and no-intervention policies. This is `off policy' estimation and while unbiased, it typically requires a large number of observations to be useful. The key technical contribution in this paper is an estimator we dub the `Dynkin estimator' that computes a treatment effect estimate directly under the steady-state distribution of the experimental policy. This estimator is biased but our main theoretical result establishes that this bias is small. Specifically, if the true treatment effect is $\delta$, then the bias of the Dynkin estimator is $O(\delta^2)$, wherein the suppressed constants depend {\em logarithmically} on the size of the state space. Our estimator relies on a relatively old idea in the RL literature, going back at least to \cite{kakadeApproximatelyOptimalApproximate2002}.
\newline
\textbf{2. Variance: }We show ({\bf Theorem \ref{th:dynkin_var}}) that the DQ estimator is asymptotically normal, and provide a non-trivial, explicit characterization of its variance. By comparison, we show ({\bf Theorem \ref{th:price}}) that this variance can, in general, be exponentially (in the size of the state space) smaller than the variance of {\em any} unbiased off-policy estimator. Our analysis introduces two new techniques. First, we prove what we dub an `Entrywise Non-expansive Lemma', that we believe is crucial to elucidating the variance reduction afforded by on-policy methods. Second, we introduce a novel linearization trick which dramatically simplifies the analysis of variance in RL via the delta method.  
%It is worth noting that the variance of off-policy estimators has not been characterized heretofore.

\noindent Summarizing the above points, we are the first (to our knowledge) to explicitly characterize the favorable bias-variance trade-off in using {\it on-policy} estimation to tackle off-policy evaluation. This new lens has broader implications for OPE and policy optimization in RL (e.g., this leads to a new approach with a provably lower bias than some widely used methods in policy optimization, see Section~\ref{sec:policy-optimization}).
%Combined with the previous point, the DQ estimator incurs a provably small bias for a massive reduction in variance.
%Implicit in the use of our estimator is the estimation of so-called $Q$-functions under a {\em fixed} policy. From an RL perspective, this is a relatively easy problem, and we are able to automatically leverage results going back to the thesis work of \cite{konda2002actor} that establish the asymptotic normality of several commonly-used estimators for this task.
\newline
\textbf{3. Practical Performance: }Despite the technical novelty described above, we view this as our most important contribution. We conduct experiments in both a caricatured one-dimensional environment proposed by others \cite{johari2022experimental}, as well as a city-scale simulator of a ride-sharing platform. We show that in both settings the DQ estimator has MSE that is substantially lower than (a) naive, and several state-of-the-art off-policy estimators, and even (b) estimators given access to incumbent state-of-the-art experimental {\em designs}.
%The remainder of this section will present a review of literature related to the Markovian interference problem and incumbent estimators. We also discuss important related tools from the Off-Policy evaluation and policy optimization literature.

\subsection{An Illustrative Example}
\label{sec:example}
It will be useful at this point to consider a simple example which highlights (a) the model of interference that we address, (b) the shortcomings of existing approaches to inference under such interference, and (c) our own approach to the problem.  {\em Importantly, all results presented in this simple example will extend to general MDPs by virtue of our analysis in Sections~\ref{sec:estimator} and \ref{sec:price}}.

Consider the continuous-time Markov chain depicted in \cref{fig:simple}; this is simply an $M/M/N/N$ queue (or the `Erlang B' model). The state space, ranging from 0 to $N$, can be thought of as the quantity of some resource (e.g.~rental homes of a similar type and geography) currently `occupied'. Customers arrive according to a Poisson process with rate $\lambda$, and independently with probability $p$, occupy a resource if available, for an exponentially-distributed duration with mean $1/\mu$. In spite of its simplicity, this model is closely related to one previously studied by  \cite{johari2022experimental} in the context of interference in commerce platforms.
\begin {figure}[!h]
\centering
\begin{tikzpicture}[->,line width =1pt, node distance =2.5cm]
	\node[circle, draw] (zero) [minimum size =14mm]{$0$};
	\node[circle, draw] (one) [right of=zero,minimum size =14mm] {$1$};
	\node[circle, draw] (two) [right of=one,minimum size =14mm]{$2$};
	\node[] (dot) [right of=two,minimum size =14mm]{...};
	\node[circle,draw] (N2) [right of=dot,minimum size =14mm]{$N-2$};
	\node[circle,draw] (N1) [right of=N2,minimum size =14mm]{$N-1$};
	\node[circle,draw] (N) [right of=N1,minimum size =14mm]{$N$};
	
	\path (zero) edge [bend left] node [above] {$p\lambda$} (one);
	\path (one) edge [bend left] node [above] {$p\lambda$} (two);
	\path (two) edge [bend left] node [above] {$p\lambda$} (dot);
	\path (dot) edge [bend left] node [above] {$p\lambda$} (N2);
	\path (N2) edge [bend left] node [above] {$p\lambda$} (N1);
	\path (N1) edge [bend left] node [above] {$p\lambda$} (N);
	
	\path (one) edge [bend left] node [below] {$\mu$} (zero);
	\path (two) edge [bend left] node [below] {$2 \mu$} (one);
	\path (dot) edge [bend left] node [below] {$3 \mu$} (two);
	\path (N2) edge [bend left] node [below] {$(N-2) \mu$} (dot);
	\path (N1) edge [bend left] node [below] {$(N-1) \mu$} (N2);
	\path (N) edge [bend left] node [below] {$N\mu$} (N1);
	
\end{tikzpicture}
\caption {A continuous-time Markov chain. Arrows indicate {\em rates} of transition between states.}
\label {fig:simple}
\end{figure}
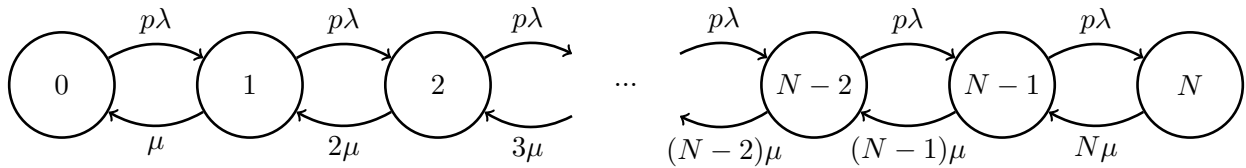

Now consider a treatment whose effect is to increase the probability $p$ by some (unknown) quantity $\delta \ge 0$. Our goal is to measure the effect of the treatment on the steady-state rate of occupation (i.e. the steady-state rate of rightward transitions). We wish to estimate the treatment effect from a simple A/B test; i.e., an experiment that randomly applies (or does not apply) the treatment to each arriving customer. We now describe the various candidate estimators under this experimental design.

% \footnote{One precise instantiation of this randomization is the following: assume that the $\mathrm{Poisson}(\lambda)$ process, which we use to represent arriving customers who will occupy an available resource, is a sub-process of another Poisson process that represents {\em all} arriving customers, some of whom choose not to occupy a resource. The effect of the treatment is on this choice likelihood, and the randomization referred to here is exactly an A/B test: each arriving customer is either treated or not treated randomly.}

\paragraph{Existing Approach 1 -- Naive:} Given the observed trajectory during this experiment, the `naive' approach measures the empirical rates at which customers with and without treatment occupy resources, and takes the difference -- effectively ignoring interference. While this is largely what is done in practice, unfortunately the resulting estimator is {\em biased}. Specifically, its expected value overestimates the true treatment effect, loosely because it ignores the fact that an increase in $p$, while increasing the immediate likelihood of occupation, has the secondary effect of {\em decreasing} the availability of resources, and thus preventing new occupations in the future. This is {\em interference}.

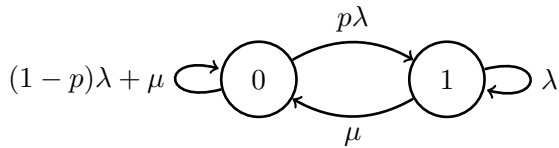
\begin {figure}[!h]
\centering
\begin{tikzpicture}[->,line width =1pt, node distance =2.5cm]
	\node[circle, draw] (zero) [minimum size =10mm]{$0$};
	\node[circle, draw] (one) [right of=zero,minimum size =10mm] {$1$};
	
	\path (zero) edge [bend left] node [above] {$p\lambda$} (one);
	\path (zero) edge [loop left] node {$(1-p)\lambda + \mu$} (zero);
	\path (one) edge [bend left] node [below] {$\mu$} (zero);
	\path (one) edge [loop right] node {$\lambda$} (one);

%	\node[circle, draw] (zeroT) [minimum size =10mm] at (9,0) {$0$};
%	\node[circle, draw] (oneT) [right of=zeroT,minimum size =10mm] {$1$};
%	
%	\path (zeroT) edge [bend left] node [above] {$(p+\delta)\lambda$} (oneT);
%	\path (zeroT) edge [loop left] node {$(1-p-\delta)\lambda + \mu$} (zeroT);
%	\path (oneT) edge [bend left] node [below] {$\mu$} (zeroT);
%	\path (oneT) edge [loop right] node {$\lambda$} (oneT);
%
%	
\end{tikzpicture}
\caption {The discrete Markov chain analogous to the continuous-time chain depicted in \cref{fig:simple}, for the case $N=1$. %{\bf Left} is the control chain, and {\bf right} is the treatment chain. 
Arrows indicate transition {\em probabilities}, rather than rates. Without loss of generality, the parameters are normalized so that $\lambda + \mu = 1$.}
\label {fig:simple2}
\end{figure}

To be concrete, consider the simplest case: $N=1$. \cref{fig:simple2} depicts the equivalent {\em discrete} Markov chain, where we have assumed (without loss of generality) that $\lambda + \mu = 1$. A new occupation occurs whenever the chain transitions from state 0 to state 1. We are interested in the rate of such transitions, which can be worked out to be $p\lambda\mu/(p\lambda + \mu)$. The additive increase in this term when $p$ is replaced with $p + \delta$, is the so-called {\em average treatment effect} (ATE) that we are after. For this example, it suffices to know that the ATE is $\Omega(\delta)$.

The chain we actually observe, e.g.~resulting from an A/B test, applies the treatment with $1/2$ probability at each time period, affecting the transition probabilities when in state 0. Given a single trajectory $\{s_t\}$ of length $T$ from this chain, the {\em Naive} estimator then is
\[\hat{\rm ATE}_{\rm NV} =  \frac{1}{|T_1|} \sum_{t\in T_1} \I{s_{t}=0,s_{t+1}=1} - \frac{1}{|T_0|} \sum_{t\in T_0} \I{s_{t}=0,s_{t+1}=1},  \]
where $T_1$ and $T_0$ are, respectively, the sets of time periods in which the treatment was and was not applied. An explicit calculation then shows 
\begin{equation} \label{eq:simple} 	
\left\lvert 
\lim_T \hat{\rm ATE}_{\rm NV}  - {\rm ATE} 
\right\rvert 
\approx 
\frac{p\lambda}{\mu} {\rm ATE},
\end{equation} 
where the approximation ($\approx$) hides terms of size $O(\delta^2)$.
In words, unless the system is extremely unoccupied ($\mu \gg p\lambda$), the Naive estimator {\em has bias that is on the order of the treatment effect}. It is worth noting that the variance of this Naive estimator is effectively $O(1)$ -- i.e. it is as small as we can hope for. 

\paragraph{Existing Approach 2 -- Off-Policy Evaluation:}
We may view the problem at hand as one of {\em Off-Policy Evaluation} in reinforcement learning. To do this, we associate an MDP with our chain. The actions in this MDP correspond to treating or not treating an arrival at a given state in the chain. The reward associated with this action is 1 if the subsequent transition is to the right and 0 otherwise. The ATE then corresponds to the difference in average steady-state rewards between the two policies which always select, respectively, the treatment and non-treatment actions.

The task of estimating the ATE is now trivially viewed as one of OPE. This in turn, immediately suggests a whole host of existing OPE estimators that yield {\em unbiased} estimates of the ATE, e.g.~\cite{jiang2016doubly,kallus2020double,kallus2022efficiently}. This natural approach appears to be promising, e.g.~\cite{shi2020reinforcement}, but outside of secondary issues (e.g.~discounted vs.~average reward), the primary issue is that being unbiased appears to come at a price: variance. Specifically, one may show that {\em any} unbiased OPE estimator has variance that grows {\it exponentially} with the number of states in our chain, as $e^{\Omega(N)}$.

This sets up two extremes of a bias-variance tradeoff in our simple example: the Naive estimator has $O(1)$ variance, but its bias is on the order of the treatment effect itself. Any unbiased OPE estimator on the other hand will have variance that scales like $e^{\Omega(N)}$.

\paragraph{Our Approach -- The DQ Estimator:} Continuing to keep in mind the MDP policy evaluation lens from above, observe that the Naive estimator effectively computes the average difference in instantaneous rewards, averaged over states visited under the policy corresponding to simple randomization. Our estimator makes one change to the Naive estimator: instead of computing the average difference in instantaneous rewards, we instead compute the average difference in {\em Q-functions} \footnote{Q-functions are formally introduced in Section~\ref{sec:estimator}.}. Intuitively, doing so allows us to partially account for the long-term effects of selecting the treatment over no-treatment at any given state, and consequently, we hope for a less biased estimate of the treatment effect.

It turns out that the DQ estimator (denoted by $\hat{{\rm{ATE}}}_{\rm{DQ}}$) provides a dramatic reduction in bias. Starting with bias, the DQ estimator's bias can be worked out explicitly here, 
\begin{align*}
\left|\lim_{T} \hat{{\rm{ATE}}}_{\rm{DQ}} - {\rm{ATE}}\right| \approx \frac{\delta}{2} \frac{\lambda}{(\mu+\lambda p)} {\rm{ATE}}
\end{align*}
from which we find that it is $O(\delta^2)$. See Appendix \ref{sec:calculation-example} for details. This is second order relative to the ATE, and, of course, a marked improvement over the Naive estimator's bias. It turns out that this reduction in bias is generic: one of our primary contributions ({\bf Theorem \ref{th:dynkin_bias}}) is to prove the DQ estimator's bias is $O(\delta^2)$ in general MDPs.

Turning next to variance, we can show that the variance of the DQ estimator in our example is $O(N)$. In contrast, an optimal unbiased estimator has variance $e^{\Omega(N)}$, so that the DQ estimator provides an exponential reduction in variance for a relatively small increase in bias. In fact, these relative merits are also generic. Specifically, in {\bf Theorem \ref{th:dynkin_var}} we upper bound the variance of the DQ estimator for general MDPs. This upper bound scales as $\log(1/\rho_{\rm min})$, where $\rho_{\rm min}$ is probability of the least-visited state under the stationary distribution, which can be as large as $1/N$; in the given example $\rho_{\rm min} = e^{-\Omega(N)}$. In {\bf Theorem \ref{th:cramer-rao}} we prove a lower bound on the variance of {\em any} unbiased estimator in the context of general MDPs, which is exponentially larger -- scaling as $\Omega(1/\rho_{\rm min})$. These two bounds show that the variance reduction relative to OPE is generic ({\bf Theorem \ref{th:price}}).

In summary, this example illustrates precisely the bias-variance trade-offs embodied by each estimator in Table~\ref{tab:bias-variance}. In particular, the DQ estimator has bias second-order in the estimand, with variance exponentially smaller than any unbiased OPE estimator --- capturing a particularly advantageous spot in the bias-variance curve. These results hold in generality for a large class of problems, which we formalize in the next section.

\begin{table}[h!]
\centering
\begin{tabular}{@{}lcc@{}} \toprule
	{\bf Estimator} & {\bf Bias} & {\bf Variance} \\ \midrule
	Naive & $\Omega(\delta)$ & $O(1)$  \\
	Off-Policy Evaluation & 0 & $e^{\Omega(N)}$ \\
	Differences-In-Q's (DQ) & $O(\delta^2)$ & $O(N)$ \\ \bottomrule
\end{tabular}
\caption{The bias-variance tradeoff of different estimators. Bias is parameterized by the additive impact $\delta$ of the intervention on transition probabilities -- note that the ATE itself can be $\Omega(\delta)$. `Variance' shows the limiting variance of each estimator on this example, as a function of the cardinality $N$ of the state space. In full generality, variance is $O(\log (1 / \rho_{\rm \min}))$ for DQ, and $\Omega(1/\rho_{\min})$ for OPE, where $\rho_{\rm min}$ is the frequency of the least-visited state under the steady-state distribution. In this example $\rho_{\rm min} = e^{-\Omega(N)}$, but in general $\rho_{\rm min}$ can be up to $1/N$.}
\label{tab:bias-variance}
\end{table}

\paragraph{Aside -- Alternative Experimental Designs: }

Whereas our focus is on estimation assuming simple-randomization, a more sophisticated {\it two-sided randomization} (TSR) design has also been studied for this specific system in \cite{johari2022experimental}. In their scheme, both customers {\it and} resources are randomized independently into treatment and control, and the intervention is applied only if both the customer and the resource are treated. We provide empirical comparisons against this approach in Section~\ref{sec:experiments}, which show that DQ outperforms TSR in typical supply / demand regimes, despite a simpler design.

\subsection{Related Literature:}

%The notion of interference in experiments can be formally traced back to at least \cite{struchiner1990behaviour,talbot1995effect,lucking-reileyUsingFieldExperiments1999} in the contexts of public medicine, farming, and online markets, respectively, and has re-appeared in multiple modern contexts including, e.g., social networks \cite{aral2011creating}, ride-sharing \cite{lyft2022experimentation}, and lodging rentals \cite{johari2022experimental}. Many of these applications fall under our notion of `Markovian' interference, whether implicitly or explicitly (e.g. \cite{qinReinforcementLearningRidesharing2021} for ride-sharing).

The largest portion of work in interference is in {\em experimental design}, with the design levers ranging from stopping times in A/B tests \cite{kharitonov2015sequential,johari2017peeking,yang2017framework,JohPekWal2020Always}, to any form of more-sophisticated `clustering' of units \cite{cornfield1978symposium,farquhar1978symposium,commit1991community,donner2004pitfalls,murray1998design,ugander2013balanced,walker2014design,eckles2017design,farias2021learning}, to clustering specifically when interference is represented by a network \cite{manski2013identification,ugander2013graph,saveski2017detecting,athey2018exact,basse2019randomization,pouget2019variance,zigler2021bipartite}, to the proportion of units treated \cite{hudgens2008toward,tchetgen2012causal,baird2018optimal}, to the timing of treatment \cite{sneider2018experiment,bojinov2020design,glynnAdaptiveExperimentalDesign2020}, and beyond \cite{backstrom2011network,katzir2012framework,toulis2013estimation,manski2013identification,choi2017estimation,basse2018model,gui2015network,saveski2017detecting,farias2022synthetically}. As alluded to earlier, these sophisticated designs can be powerful, but cost, user experience, and other implementation concerns restrict their application in practice \cite{kirnChallengesExperimentation2022,kohavi2020trustworthy}.

We view this paper as orthogonal to this literature, but will eventually compare against a recent state-of-the-art design, so-called {\em two-sided randomization} \cite{johari2022experimental,bajari2021multiple}, that is specific to the context of two-sided marketplaces (e.g.~the one we simulate).

As stated earlier, the problem we study is one of {\em off-policy evaluation (OPE) }\cite{precup2001off,sutton2008convergent}. %, itself a classic problem in reinforcement learning.
The fundamental challenge in OPE %, which we view as one reason why we cannot simply port one of the many existing unbiased estimators to our problem,
is high variance, which can be attributed to the nature of the algorithmic tools used, e.g. sampling procedures \cite{thomas2015high,thomas2016data,liu2018breaking}. Recent work on `doubly-robust' estimators \cite{jiang2016doubly,kallus2020double,kallus2022efficiently} has improved on variance (incidentally, our estimator is loosely tied to these, as we discuss in Section \ref{sec:conclusion}), but again we will show, via a formal lower bound, that unbiased estimators as a whole have prohibitively large variance. Finally, our motivation is close in spirit to a recent paper \cite{shi2020reinforcement}, which applies OPE directly in Markovian interference settings; we make a direct experimental comparison in Section \ref{sec:experiments}.

In the policy optimization literature, `trust-region' methods \cite{schulman2015trust} and conservative policy iteration \cite{kakadeApproximatelyOptimalApproximate2002} use a related on-policy estimation approach to bound policy improvement. Relative to the existing literature, we develop an on-policy surrogate with provably lower bias than extant proposals; see Section~\ref{sec:policy-optimization}. Furthermore, the explicit application of on-policy estimation in the context of OPE, and in particular the striking bias-variance tradeoff this enables, are novel to this paper.

%%% Local Variables:
%%% mode: latex
%%% TeX-master: "main.tex"
%%% End:

\section{Model} \label{sec:model}
Having discussed each estimator in a specific example, we now formalize the general inference problem that we tackle, casting it in the language of MDPs. Vis-\`{a}-vis the existing literature, this lens allows us to reason about the problem using a large, well-established toolkit, and makes obvious the fact that OPE provides unbiased estimation of the ATE. We then present what we call the `Naive' estimator (alluded to in the introduction). This is the lowest-variance estimator one can hope for in this setting, but it can have significant bias, as we see in \cref{eq:simple}.

% We then briefly discuss two state-of-the-art alternatives to the Naive estimator: the TSR estimator \cite{johari2022experimental} that supposedly mitigates bias while also being low variance, and Off-Policy evaluation (which, thanks to our problem framing is an obvious unbiased alternative). Finally, we present our proposal -- the DQ estimator.

We begin by defining an MDP with state space $\Sscr$. We denote by $s_t \in \Sscr$ the state of the MDP at time $t \in \mathbb{N}$. Every state is associated with a set of available actions $\Ascr$ which govern the transition probabilities between states via the (unknown) function $p:\Sscr \times \Ascr \times \Sscr \to [0,1]$. We assume that $\Ascr = \{0,1\}$ irrespective of state; for descriptive purposes, we will associate the `$1$' action with the use of a prospective intervention, so that `$0$' is associated with not employing the intervention. We denote by $r(s,a)$ the reward earned in state $s$ having employed action $a$. A policy $\pi: \Sscr \rightarrow \Ascr$ maps states to random actions. We define the average reward $\lambda^\pi$, under any (ergodic, unichain) policy $\pi$, according to:
\[
\lambda^\pi = \lim_{T \to \infty} \frac{1}{T} \sum_{t=1}^T r(s_t,\pi(s_t)).
\]
There are three policies we define explicitly: 

\textbf{The Incumbent Policy $\pi_0$: }This policy never uses the intervention, so that $\pi_0(s) = 0$ for all $s$. This is `business as usual'. Denote the associated transition matrix as $P_{0}$ (i.e. the entries of $P_0$ are exactly $p(\cdot,0,\cdot)$)

\textbf{The Intervention Policy $\pi_1$: }This policy always uses the intervention, so that $\pi_1(s) = 1$ for all $s$. This reflects the system, should the intervention under consideration be `rolled out'. Denote the associated transition matrix as $P_{1}$.

\textbf{The Experimentation Policy $\pi_{p}$: }This policy corresponds to the experiment design. Simple randomization would select $\pi(s) = 1$ with some fixed probability $p$, say $1/2$, independently at every period. This corresponds to the sort of search engine experiment alluded to in the introduction. The transition matrix associated with this design is then $P_{1/2} = \frac{1}{2}P_{0} + \frac{1}{2}P_{1}$.

% Relative to the sort of search engine experiment alluded to in the introduction, the state of the above MDP might codify the available inventory of various products on the platform, with each time step corresponding to a user search, and transitions in the state driven by potential purchases. A much more careful modeling exercise is presented in Section~\ref{sec:experiments} in the context of ride-hailing. We are now in a position to state the inference problem we will tackle:
%
\noindent\textbf{\textsf{The Inference Problem: }}We are given a single sequence of $T$ states, actions, and rewards, observed under the experimentation policy $\pi_p$ (recall that cost and constraints \cite{kirnChallengesExperimentation2022,kohavi2020trustworthy} prohibit us from running $\pi_{0}$ or $\pi_{1}$ separately until convergence). We observe the sequence  $\{(s_t, a_t, r(s_t,a_t)): t = 1,\dots,T\}$, wherein $a_t \triangleq \pi_p(s_t)$. Our goal is to estimate the average treatment effect (ATE): ${\rm ATE} \triangleq \lambda^{\pi_1} - \lambda^{\pi_0}. $

\noindent\textbf{\textsf{The Naive Estimator and Bias.}} A natural approach to estimating the ATE is to use simple randomization (i.e.~$P_{1/2}$) and  the {\em Naive} estimator, which we define in the language of MDPs as: $\hat{\rm ATE}_{\rm NV} =  \frac{1}{|T_1|} \sum_{t\in T_1} r(s_t,a_t) - \frac{1}{|T_0|} \sum_{t\in T_0} r(s_t,a_t),$
where $T_1 = \{t: a_t =1\}$ and $T_0 = \{t: a_t =0\}$. In the context of the example of Section~\ref{sec:example}, this corresponds to simply taking the difference between the probability of renting a resource among test users ($T_1$), and control users ($T_0$). What goes wrong is simply that the two empirical averages above, that seek to estimate $\lambda^{\pi_1}$ and $\lambda^{\pi_0}$ respectively, employ the wrong measure over states.
%In particular, $\lambda^{\pi_1}$ ought to be calculated under the stationary distribution induced by $\pi_1$, and $\lambda^{\pi_0}$ ought to be calculated under the stationary distribution induced by $\pi_0$, but the estimator above attempts to calculate both these averages under the stationary distribution induced by the policy $\pi_p$.
As we saw, this is sufficient to introduce bias that is on the order of the treatment effect being estimated.
\section{The Differences-In-Q's Estimator}
\label{sec:estimator}

We are now prepared to introduce our estimator for inference in the presence of Markovian interference.
%Our choice of name rests on the fact that calculating the bias of this estimator (which we show to be small) relies on Dynkin's formula, an identity that has been rediscovered many times in the RL literature. 
Before defining our estimator, which we will see is only slightly more complicated than the Naive estimator, we recall a few useful objectis in average-reward MDPs. Denote the average cost of a policy $\pi$ by $\lambda^{\pi}$. The $V$-function of a policy $\pi$, $V_{\pi}$, characterizes the ``reward-to-go'' $V_{\pi}(s) := \E\Big[\sum_{t=0}^{\infty} r(s_t, a_t) - \lambda^{\pi}~\big|~s_0 = s\Big].$
% Do we need to define this anywhere?
It is also known that $(V_{\pi}, \lambda^{\pi})$ is the fixed point of the Bellman operator $T_{\pi}$ with $T_{\pi}(V_{\pi}, \lambda^{\pi}) = V_{\pi}.$ Here $T_{\pi}: \mathbb{R}^{|\Sscr|} \times \mathbb{R} \rightarrow \mathbb{R}^{|\Sscr|}$ is given by $T_\pi (V, \lambda)  = r_\pi - \lambda \mathbf{1} + P_{\pi} V$ where $r_\pi: \Sscr \rightarrow \mathbb{R}$ is defined according to $r_\pi(s) = \E\left[r(s, \pi(s))\right]$.
%The average cost of policy $\pi$, denoted $\lambda^\pi$, and the bias function corresponding to $\pi$, denoted $V_{\pi}$, are then a solution to the fixed point equation $T_\pi (V, \lambda)  = V$. 
Finally, the $Q$-function associated with $\pi$, denoted $Q_\pi: \Sscr \times \Ascr \rightarrow \mathbb{R}$, is defined according to $Q_\pi(s,a) := \E\Big[\sum_{t=0}^{\infty} r(s_t, a_t) - \lambda^{\pi}~\big|~s_0 = s, a_0 = a\Big]$. Put simply, the $Q$-function measures the `excess' reward obtained starting from $s$ with the action $a$ relative to the average reward under $\pi$. 

\subsection{An Idealized First Step}
In motivating our estimator, let us begin with the following idealization of the Naive estimator, where we denote by $\rho_{1/2}$ the steady state distribution under the randomization policy $\pi_{1/2}$: $\E_{\rho_{1/2}} \left[\hat {\rm ATE}_{\rm NV}\right] = \sum_s \rho_{1/2}(s) \left[ r(s,1) - r(s,0)\right].$ It is not hard to see that in the example of Section~\ref{sec:example}, we continue to have $\lvert \E_{\rho_{1/2}} [\hat {\rm ATE}_{\rm NV} ] - {\rm ATE} \rvert \approx \frac{p\lambda}{\mu} \mathrm{ATE}$, i.e. this idealization of the Naive estimator continues to have bias on the order of the treatment effect. Consider then, the following alternative:
\[
\E_{\rho_{1/2}} \left[\tauDK\right] = \sum_s \rho_{1/2}(s) \left[Q_{\pi_{1/2}}(s,1) - Q_{\pi_{1/2}}(s,0)\right],
\]
where the term $\E_{\rho_{1/2}} [\tauDK]$ can for now just be thought of as an idealized constant ($\tauDK$ is defined soon in \eqref{eq:Dynkin}).
Compared to $\E_{\rho_{1/2}} [\hat {\rm ATE}_{\rm NV}]$, we see that $\E_{\rho_{1/2}} [\hat {\rm ATE}_{\rm DQ}]$  takes a remarkably similar form, except that as opposed to an average over differences in rewards, we compute an average of differences in $Q$-function values. The idea is that doing so will hopefully compensate for the shift in distribution induced by $\pi_{1/2}$, as it does in the example of Section~\ref{sec:example}.
%{\em Crucially, we rely on the $Q$-function  corresponding to $\pi_{1/2}$}. Since the experiment generates data under this very policy, it tuns out that these $Q$-values are substantially easier to estimate. 
% \begin{example}[Continued]
% Continuing with our example, we can explicitly calculate $Q_{\pi_{1/2}}(\cdot,\cdot)$, the average reward $\lambda^{\pi_{1/2}}$, and the stationary distribution $\rho_{1/2}$  (see Appendix \ref{sec:calculation-example}). Doing so allows us to calculate that
% \[
% \E_{\rho_{1/2}} \left[
% \hat {\rm ATE}_D
% \right]
% =
% \frac{1}{2} \left(
% \frac{\delta}
% {(1 - \delta/2)^2}
% \right).
% \]
% That is, $|{\rm ATE} - \E_{\rho_{1/2}} [\hat {\rm ATE}_D]| = O(\delta^2)$,
% so that the bias of this idealized estimator is second-order (i.e. negligible) relative to the {\em ATE}.
% \end{example}

Is the dramatic mitigation of bias we see in the example generic? If the experimentation policy mixes fast, our first set of results essentially answers this question in the affirmative. In particular, we make the following mixing time assumption:

%\begin{assumption}
%[Transitions]
%We assume that the change in distribution over next states under the intervention is Lipschitz in the treatment effect. Specifically, for any state $s \in \Sscr$,  
%\[
%d_{\rm TV}( p(s,1,\cdot), p(s,0,\cdot) ) \leq L \delta 
%\]
%where $L$ is a Lipschitz constant and $\delta$ is the $ATE$.   
%\end{assumption} 
%The above assumption is relatively benign. In particular, the condition is trivially satisfied for the examples in $\cite{}$. Our first result concerns the bias of the Dynkin estimator, which we show is small: 

\begin{assumption}[Mixing time]\label{assump:mixing-time}
There exist constants $C$ and $\lambda$ such that for all $s \in \Sscr$, $d_{\rm TV} (P_{1/2}^{k}(s, \cdot), \rho_{1/2}) \leq C \lambda^{k}$ where $d_{\rm TV}(\cdot,\cdot)$ denotes total variation distance.
\end{assumption}

We then have that the second order bias we saw in Section~\ref{sec:example} is, in fact, generic:

\begin{theorem}[Bias of DQ]
\label{th:dynkin_bias}
Assume that for any state $s \in \Sscr$,  
$
d_{\rm TV}( p(s,1,\cdot), p(s,0,\cdot) ) \leq \delta. 
$
Then,
\[
\left|{\rm ATE} - \E_{\rho_{1/2}} \left[
\hat {\rm ATE}_{\rm DQ}
\right]\right| \leq C' \left(\frac{1}{1-\lambda}\right)^2 r_{\max} \cdot \delta^2
\]
where $r_{\max} := \max_{s,a} \left|r(s,a)\right|$ and $C'$ is a constant depending (polynomially) on $\log(C)$.
\end{theorem}
%\subsubsection{Proof outline of Theorem \ref{th:dynkin_bias}}
% The proof of \cref{th:dynkin_bias} relies on the following Taylor expansion of the average reward around $P_{1/2}$:
% \begin{lemma}[Taylor expansion of the average reward.]
%   \label{lem:taylor}
%   For any $K \in \mathbb{N}$, it holds that
%   \begin{equation*}
%     \lambda^{\pi_{1}} = \sum_{k=0}^{K} \delta^{k} \rho_{1 / 2}^{\top} \left( \Delta (I - P_{\pi_{1 / 2}})^{\#} \right)^{k} r_{1} + o(\delta^{K})
%   \end{equation*}
%   where $r_{1}$ is the vector $\{r(s, 1)\}_{s \in \mathcal{S}}$, and $\delta\Delta = P_{\pi_{1}} - P_{\pi_{1 / 2}} $.
% \end{lemma}
% This Lemma immediately motivates both the Naive estimator

\subsection{The Differences-In-Q's Estimator}

Motivated by the development in the previous subsection, the {\em Differences-In-Q's (DQ)} estimator we propose to use is simply 
\begin{equation}
\label{eq:Dynkin}
\hat {\rm ATE}_{\rm DQ}
= 
\frac{1}{|T_1|}
\sum_{t \in T_1}
\hat Q_{\pi_{1/2}}(s_t,a_t)
-
\frac{1}{|T_0|}
\sum_{t \in T_0}
\hat Q_{\pi_{1/2}}(s_t,a_t),
\end{equation}
where we take an empirical average over the state trajectory produced under the randomization policy, and $\hat Q_{\pi_{1/2}}$ is an estimator of the $Q$-function. For concreteness, we obtain $\hat Q_{\pi_{1/2}}$ by solving
\begin{align}\label{eq:LSTD0-V}
\min_{\hat V, \hat \lambda}
\sum_{s \in \Sscr} 
\left(\sum_{t, s_{t}=s} 
r(s_t, a_t) - \hat{\lambda} + \hat V(s_{t+1}) - \hat V(s_t)
\right)^2.
\end{align}
Our main result characterizes the variance and asymptotic normality of $\hat {\rm ATE}_{\rm DQ}$:
\begin{theorem}[Variance and Asymptotic Normality of DQ]
\label{th:dynkin_var}
The DQ estimator is asymptotically normal so that $
\sqrt{T}
\left(\hat {\rm ATE}_{\rm DQ} - \E_{\rho_{1/2}} \left[\hat {\rm ATE}_{\rm DQ}\right]\right) \overset{d}{\rightarrow} \mathcal{N}(0, \sigma_{\rm DQ}^2),
$
with limiting standard deviation 
\[
\sigma_{\rm DQ} \leq C'\left(\frac{1}{1-\lambda}\right)^{5/2} \log \left( \frac{1}{\rho_{\min}}\right) r_{\max}.
\]
where $\rho_{\min} := \min_{s\in S} \rho_{1/2}(s)$ and $C'$ is a constant depending (polynomially) on $\log(C)$. 
\end{theorem}

The fact that $\sigma_{\rm DQ}$ in \cref{th:dynkin_var} only depends on $1/\rho_{\min}$ logarithmically is somewhat surprising. In fact, a coarse analysis will lead to $\sigma_{D} = \Omega\left(\frac{1}{\rho_{\min}}\right)$, which shows no advantage compared to the unbiased OPE estimators (which we will see momentarily). The key enabler for this striking result is a novel lemma that exploits an \textit{entry-wise bound} for controlling the variance, even at states that are rarely visited (we dub this the ``Entry-wise Non-expansive Lemma''; see \cref{lem:non-expansive-lemma}). The lemma admits a simple form and may have broader implications for analyzing variance in OPE estimators (see Discussions in \cref{sec:conclusion}). In addition, our asymptotic normality analysis borrows the delta-method framework used in the context of on-policy LSTD \cite{konda2002actor}, but with a novel linearization that dramatically simplifies the analysis. See \cref{sec:proof-sketch-dynkin-var} for more details. 

\noindent\textbf{\textsf{One Extreme of the Bias-Variance Tradeoff:}} We may heuristically think of the Naive estimator as representing one extreme of the bias-variance tradeoff among reasonable estimators. For the sake of comparison, by the Markov Chain CLT, the Naive estimator is also asymptotically normal with standard deviation $\Theta (r_{\rm max}/ (1-\lambda)^{1/2})$. This rate is efficient for the estimation of the mean of a Markov chain \cite{greenwood1995efficiency}. On the other hand, while the Naive estimator is effectively useless for the problem at hand given its bias is in general $\Theta(\delta)$, that of the DQ estimator is $O(\delta^2)$.

\subsection{Proof of \cref{th:dynkin_bias}}\label{sec:proof-of-dynkin-bias}
The proof of \cref{th:dynkin_bias} is a simple proof built on a perturbation formula for stationary distributions of Markov chains. We in fact construct a novel Taylor series representation of the ATE parameterized by $\delta$ that controls the perturbation around $P_{1 / 2}$, which yields the Naive estimator as the zeroth-order truncation of the series; and the idealized DQ estimator as the natural first-order correction. \cref{th:dynkin_bias} then proceeds by bounding the remainder. This strategy additionally allows us to generalize the DQ estimator to arbitrarily high-order bias corrections, by computing $Q$-functions iteratively. Here we present the proof (with some details omitted for simplicity). %These results can be of independent interest for the field of OPE and policy optimization. See proof details in \cref{sec:proof-of-dynkin-bias}.
%The proof is a simple proof based on observing a particular way of parameterizing ATE as a function of $\delta$. Turns out the Taylor expansion of such a function reveals naturally the bias-correction results of corresponding estimators (Naive estimator is the zeroth-order Taylor truncation, the DQ estimator is the first-order Taylor truncation, et.al). 

We first define few pieces of useful notation. Let $\rho_{0} \in \R^{|\Sscr|}, \rho_{1/2} \in \R^{|\Sscr|}, \rho_{1} \in \R^{|\Sscr|}$ be the vectors of the stationary distributions of $P_0, P_{1/2}, P_{1}$ accordingly. Let $r_0 \in \R^{|\Sscr|}, r_{1/2} \in \R^{|\Sscr|}, r_1 \in \R^{|\Sscr|}$ be the reward vectors associated with policies $\pi_0, \pi_{1/2}, \pi_{1}$, i.e., $r_{a}(s) = r(s, a)$ and $r_{1/2} = \frac{1}{2} r_0 + \frac{1}{2} r_1.$

To begin, we parameterize $P_0 := P_{1/2} - \delta A$ and $P_1 := P_{1/2} + \delta A$ by $\delta$ with fixed $P_{1/2}$ and some fixed matrix $A \in \R^{|\Sscr|\times |\Sscr|}$ with $\|A\|_{1,\infty} \leq 1$ ($\|A\|_{1,\infty} = \max_{i}\sum_{j}|A_{ij}|$)\footnote{This is always possible since $d_{\rm TV}( p(s,1,\cdot), p(s,0,\cdot) ) \leq \delta.$ }. Then, $\rho_{0}$ and $\rho_{1}$ can also be viewed as a function of $\delta.$ Also recall ${\rm{ATE}} = \rho_{1}^{\top} r_{1} - \rho_{0}^{\top} r_0$. Our goal is to represent ${\rm{ATE}}$ as a function of $\delta$ and then study the Taylor expansion of such a function. To do so, we use the following known perturbation formula of Markov chains.
\begin{lemma}[Stationary Distribution Perturbation, Theorem 4.1 \cite{meyer1980condition}]\label{lem:perturbation-stationary-distribution}
Suppose $P \in \R^{n \times n}$ and $P' \in \R^{n\times n}$ are transitions matrices of two finite-state aperiodic and irreducible Markov Chains and $\rho \in \R^{n}, \rho' \in \R^{n}$ are the stationary distributions accordingly. Then $\rho'^{\top} = \rho^{\top} + \rho'^{\top} (P' - P) (I - P)^{\#}$ 
where $(I-P)^{\#}$ is the group inverse of $I-P$ given by $(I-P)^{\#}=(I-P+\one \rho^{\top})^{-1} - \one \rho^{\top}.$
\end{lemma}
Let us apply \cref{lem:perturbation-stationary-distribution} to $\rho_{1}^{\top} r_1$ based on the perturbation between $\rho_{1/2}$ and $\rho_{1}$.
\begin{align}
\rho_1^{\top} r_1 
&= \rho_{1/2}^{\top} r_1 + \rho_{1}^{\top} (P_1 - P_{1/2}) (I-P_{1/2})^{\#} r_1 \nonumber\\
&= \rho_{1/2}^{\top} r_1 + \delta \cdot \rho_{1}^{\top} A (I-P_{1/2})^{\#} r_1 \label{eq:expansion-rho-rho1}
\end{align}
Note that we can apply \cref{lem:perturbation-stationary-distribution} again to the $\rho_{1}$ in the RHS of \cref{eq:expansion-rho-rho1} and then repeat this process, 
\begin{align}
\rho_1^{\top} r_1 
&= \sum_{k=0}^{K} \delta^{k} \cdot \rho_{1/2}^{\top} \left(A (I-P_{1/2})^{\#}\right)^k r_1 + \delta^{K+1} \cdot \rho_{1}^{\top} \left(A (I-P_{1/2})^{\#}\right)^{K+1} r_1 \label{eq:expansion-rho-rho1-K}
\end{align}
for any $K = 0, 1, 2, \dotsc$. Essentially \cref{eq:expansion-rho-rho1-K} provides the $K$-th order Taylor expansion for $\rho_{1}^{\top}r_1$ with an explicit remainder. Furthermore, we can bound the remainder by
\begin{align*}
\left|\rho_{1}^{\top} \left(A (I-P_{1/2})^{\#}\right)^{K+1} r_1\right| 
&\overset{(i)}{\leq} \norm{\rho_{1}}_{1} \left(\norm{A}_{1,\infty} \norm{I-P_{1/2}^{\#}}_{1,\infty}\right)^{K+1} \norm{r_1}_{\max}\\
&\overset{(ii)}{\leq} \norm{I-P_{1/2}^{\#}}_{1,\infty}^{K+1} r_{\max} \\
&\overset{(iii)}{\leq} \left(\frac{2\ln(C)+1}{1-\lambda}\right)^{K+1} r_{\max}
\end{align*}
Here in (i) we use that for any vector $a, b$ and matrix $B$, we have $|a^{\top}b| \leq \|a\|_{1}\|b\|_{\max}$ and $\norm{a^{\top}B}_{1} \leq \norm{a}_{1} \norm{B}_{1,\infty}$. In (ii) we use that $\norm{\rho_1}_1 = 1, \norm{A}_{1,\infty} \leq 1$. In (iii), we use the following lemma implied by the mixing time assumption and the series expansion of $(I-P)^{\#}.$
\begin{lemma}\label{lem:group-inverse-L1}
    Suppose for any $s \in \Sscr$, $d_{\rm TV} (P^{k}_{1/2}(s, \cdot), \rho_{1/2}) \leq C \lambda^{k}$. Then $\norm{(I-P_{1/2})^{\#}}_{1,\infty} \leq \frac{2\ln(C)+1}{1-\lambda}.$
\end{lemma}

Appplying a similar process to $\rho_{0}^{\top}r_0$, we obtain the Taylor expansion for the ATE.
\begin{align}
{\rm{ATE}} 
%&= \rho_1^{\top} r_1 - \rho_0^{\top} r_0 \nonumber\\
&=  \sum_{k=0}^{K} \delta^{k} \cdot \left(\rho_{1/2}^{\top} \left(A (I-P_{1/2})^{\#}\right)^k r_1 - \rho_{1/2}^{\top} \left((-A) (I-P_{1/2})^{\#}\right)^k r_0\right) + \delta^{K+1} \cdot a_{K} \label{eq:taylor-expansion-ATE}
\end{align}
where $|a_{K}| \leq 2\left(\frac{2\ln(C)+1}{1-\lambda}\right)^{K+1} r_{\max}.$ It is easy to see that the Naive estimator $\rho_{1/2}^{\top}(r_1-r_0)$ corresponds to the zeroth-order truncation. In fact, the DQ estimator, i.e., $\E_{\rho_{1/2}} \left[\tauDK \right]$, exactly matches the first-order truncation. To see this, by the definition of $\E_{\rho_{1/2}} \left[\tauDK \right]$ and $Q$-functions, 
\begin{align*}
\E_{\rho_{1/2}} \left[\tauDK \right] 
&= \sum_s \rho_{1/2}(s) \left(Q_{\pi_{1/2}}(s,1) - Q_{\pi_{1/2}}(s,0) \right) \\
&= \sum_{s} \rho_{1/2}(s) \left(r_1(s) +\sum_{s'} V_{1/2}(s')P_{1}(s,s') - r_0(s) - \sum_{s'} V_{1/2}(s')P_{0}(s,s')\right)\\
&= \rho_{1/2}^{\top} \left(r_1 - r_0 + (P_1-P_0)V_{1/2}\right)
\end{align*}
where $V_{1/2}$ is the induced vector of the $V$-function of policy $\pi_{1/2}.$ By the well-known fact that $V_{1/2} = (I-P_{1/2})^{\#} r_{1/2}$ induced by the Bellman equation, we then have
\begin{align*}
\E_{\rho_{1/2}} \left[\tauDK \right] 
&=  \rho_{1/2}^{\top} \left(r_1 - r_0 + (P_1-P_0)(I-P_{1/2})^{\#} r_{1/2}\right)\\
&= \rho_{1/2}^{\top} r_1  - \rho_{1/2}^{\top} r_0 + \delta \rho_{1/2}^{\top}A (I-P_{1/2})^{\#} (r_{1}+r_{0}).
\end{align*}
Then indeed $\E_{\rho_{1/2}} \left[\tauDK \right]$ is the first-order Taylor truncation. Together, this completes the proof.
%\begin{align*}
%\left|{\rm{ATE}} - \E_{\rho_{1/2}} \left[\tauDK \right]\right| \leq 2\delta^2 r_{\max}\left(\frac{2\ln(C)+1}{1-\lambda}\right)^2.
%\end{align*}
\noindent\textbf{\textsf{Generalization to Higher-Order Bias Correction.}} In fact, the K-th order Taylor expansion of ATE allows us to design estimators that can correct higher-order bias, based on computing difference-in-Q functions iteratively. See details in \cref{sec:korder-correction}.

\subsection{Proof Sketch of \cref{th:dynkin_var}}\label{sec:proof-sketch-dynkin-var}
We aim to use the Markov chain CLT (\cite{jones2004markov}) to show asymptotic normality of our estimator. The Markov chain CLT states that for a Markov chain $X_{1}, X_{2}, \dotsc, $ and a bounded function $u$ with domain on the state space, there exists $\Sigma_{u}$ such that $\sqrt{T}\left(\frac{1}{T}\sum_{t=1}^{T} u(X_t) - u^{*}\right) \overset{d}{\rightarrow} N(0, \Sigma_{u})$
where $u^{*}$ is the expected value of $u$ under the stationary distribution of the Markov chain. See proof details in \cref{sec:proof-variance-bound}.

\noindent\textbf{\textsf{Delta method.}} Unfortunately, the estimator $\tauDK$ can not be directly written as an empirical average of some function $u.$ To address this issue, we use the the delta method (traced back to \cite{doob1935limiting}, see \cref{lem:mapping-CLT}). In particular, we write $\tauDK = f(u_T)$ as a function of a random vector $u_{T}$ given by $u_{T} := \frac{1}{T}\sum_{t=1}^{T} u(X_{t}).$ Under some minor conditions, the delta method states that $\sqrt{T} \left(f(u_{T}) - f(u^{*})\right) \overset{d}{\rightarrow} N(0, \sigma_{f}^2)$
where $\sigma_{f}^2 := \nabla f(u^{*})^{\top} \Sigma_{u} \nabla f(u^{*})$ and $\nabla f(u^{*})$ is the gradient of $f$ evaluating at the point $u^{*}.$ This forms the basis for proving \cref{th:dynkin_var}.

\noindent\textbf{\textsf{Linearization.}} To simplify the analysis for $\sigma_{f}$, instead of computing $\Sigma_{u}$ explicitly, we ``linearize'' the function $f$ by defining $\tilde{f}(X_{t}) := \nabla f(u^{*})^{\top}(u(X_t)-u^{*})$ and the delta method in fact implies (see \cref{lem:linearization}) $\sqrt{T}\left(\frac{1}{T}\sum_{t=1}^{T} \tilde{f}(X_t)\right) \overset{d}{\rightarrow} N(0, \sigma_{f}^2),$
i.e., the linearized $f$ converges with the same limiting variance as the original $f.$ Therefore, we can focus on $\tilde{f}$ for analyzing $\sigma_{f}.$ 

\noindent\textbf{\textsf{Bounding $\sigma_f$ with Entry-wise Non-expansive Lemma.}} To bound $\sigma_{f}$, we will invoke \cref{lem:bound-for-sigma}, which states that $\sigma_{f} \leq \sqrt{2}\sqrt{\frac{2\ln(C)+1}{1-\lambda}}\tilde{f}_{\max}$
where $\tilde{f}_{\max} := \max_{s} |\tilde{f}(s)|.$ Then the problem reduces to bounding $\tilde{f}_{\max}$, which will be controlled by the following key lemma.
\begin{lemma}[Entry-wise non-expansive lemma]\label{lem:non-expansive-lemma}
Let $W: \R^{|\Sscr|} \rightarrow \R^{|\Sscr|}$ be a map denoted by $W(\rho) := (I-P_{1/2})^{\#\top}(P_1-P_0)^{\top} \rho.$ Let $c := 4\frac{\ln(C) +  \ln \left(1/\rho_{\min}\right) + 1}{1-\lambda}.$ Then, for any $s \in \Sscr$,  $\frac{1}{c}\left|W(\rho_{1/2}) (s)\right| \leq \rho_{1/2}(s).$
\end{lemma}

\section{The Price of Being Unbiased} \label{sec:price}
Thus far, we have seen that the DQ estimator provides a dramatic mitigation in bias (Theorem~\ref{th:dynkin_bias}) at a relatively modest price in variance (Theorem~\ref{th:dynkin_var}). This suggests another question: could we hope to construct an {\em unbiased} estimator that has low variance (i.e. comparable to either the Naive or DQ estimators). We will see that the short answer is: no.

\subsection{The Variance of an Optimal Unbiased Estimator}

As noted earlier, a plethora of Off-policy evaluation (OPE) algorithms might be used to provide an unbiased estimate of the ATE. Rather than consider a particular OPE algorithm, here we produce a Cram\'{e}r-rao lower bound on the variance of {\em any} unbiased OPE algorithm. While such a bound is obviously of independent interest (since OPE is a far more general problem than what we seek to accomplish in this paper), we will primarily be interested in comparing this lower bound to the variance of the DQ estimator from Theorem~\ref{th:dynkin_var}.

\begin{theorem}[Variance Lower Bound for Unbiased Estimators]
\label{th:cramer-rao}
Assume we are given a dataset  $\{(s_t,a_t, r(s_t,a_t)): t=0, \dots,T \}$ generated under the experimentation policy $\pi_{1/2}$, with $s_0$ distributed according to $\rho_{1/2}$.  Then for any unbiased estimator $\hat{\tau}$ of $\mathrm{ATE}$, we have that
\begin{align*}
T\cdot \mathrm{Var}(\hat{\tau}) 
&\geq 2\sum_{s} \frac{\rho_1(s)^2}{\rho_{1/2}(s)} \sum_{s'} p(s, 1, s')(V_{\pi_1}(s') - V_{\pi_1}(s) + r(s,1) - \lambda^{\pi_1})^2\\
&\quad + 2\sum_{s} \frac{\rho_0(s)^2}{\rho_{1/2}(s)} \sum_{s'} p(s, 0, s')(V_{\pi_0}(s') - V_{\pi_0}(s) + r(s,0) - \lambda^{\pi_0})^2 \triangleq \sigoff^2.
\end{align*}
\end{theorem}

It is worth remarking that this lower bound is tight: in the appendix we show that an LSTD(0)-type OPE algorithm achieves this lower bound. While this is of independent interest vis-\`{a}-vis 
average cost OPE, we turn next to our ostensible goal here -- evaluating the `price' of unbiasedness. We can do so simply by comparing the variance of the DQ estimator with the lower bound above. In fact, we are able to exhibit a class of one-dimensional Markov chains (in essence the model in \cref{sec:example}) for which we have: 

\begin{theorem}[Price of Unbiasedness]
\label{th:price}
    For any $0 < \delta \leq \frac{1}{5}$, there exists a class of MDPs parameterized by $n \in\mathbb{N}$, where $n$ is the number of states, such that   $\frac{\sigma_{\rm DQ}}{\sigoff} = O\left(\frac{n}{c^{n}}\right),$
    for some constant $c>1$. Furthermore, $|(\mathrm{ATE}-\E[\hat{\mathrm{ATE}}_{\rm DQ}])/\mathrm{ATE}| \leq \delta.$
\end{theorem}

\noindent\textbf{\textsf{Another Extreme of the Bias-Variance Tradeoff:}} Theorems~\ref{th:dynkin_var}, \ref{th:cramer-rao}, and \ref{th:price} together reveal the opposite extreme of the bias-variance tradeoff. Specifically, if we insisted on an unbiased estimator for our problem (of which there are many, thanks to our framing of the problem as one of OPE), we would pay a large price in terms of variance. In particular Theorem~\ref{th:price} illustrates that this price can grow exponentially in the size of the state space. This jibes with our empirical evaluation in both caricatured and large-scale MDPs in Section~\ref{sec:experiments}. 

Taken together our results reveal that the DQ estimator accomplishes a striking bias-variance tradeoff: it has substantially smaller variance than any unbiased estimator (in fact, comparable to the Naive estimator), all while ensuring bias that is second order in the impact of the intervention.

%%% Local Variables:
%%% mode: latex
%%% TeX-master: "main.tex"
%%% End:

%
%\input{ResultB.tex}
%
%\input{Inference.tex}

\section{Experiments} \label{sec:experiments}
This section will empirically investigate the DQ estimator and a number of alternatives in two settings: the simple example of Section~\ref{sec:example}, originally
proposed by \cite{johari2022experimental}; and more realistically, a city-scale simulator of a ride-hailing platform similar to what large ride-hailing operators use in production. The alternatives we consider include: 1) the Naive estimator;  2) TSRI-1 and TSRI-2, the ``two-sided randomization'' (TSR) designs/estimators from \cite{johari2022experimental}; and 3) a variety of OPE estimators. For the OPE estimators, we note that off-policy average reward estimation has only recently been addressed in \cite{wanLearningPlanningAverageReward2021,zhangAverageRewardOffPolicyPolicy2021}, and we implement their specific estimators which we simply denote as TD and GTD respectively. We also implement an extension to an LSTD type estimator proposed in \cite{shi2020reinforcement}.
\subsection{A Simple Example}
We first study all of our estimators in the example of Section~\ref{sec:example}, a simple setting that does not call for any sort of value function approximation. Our goal now is to understand the relative merits of practical implementations of these estimators, in terms of their bias and variance.

To recap, this MDP is a stylized model of a rental marketplace, consisting of a 1-D Markov chain on $N=5000$ states parameterized by a `customer arrival' rate $\lambda$ and a `rental duration' rate $\mu$. At a given state $n$ (so that $n$ units of inventory are in the system), the probability that an arriving customer rents a unit is impacted by the intervention. As such if the intervention increases the probability of a customer renting, this reduces the inventory availability for customers that arrive later. Our MDP and experimental setup exactly replicates that of \cite{johari2022experimental}, with $N=5000, \lambda=1, \mu=1$.
\begin{figure}[htbp]
 \centering
 \subfloat{\includegraphics[width=0.46\textwidth]{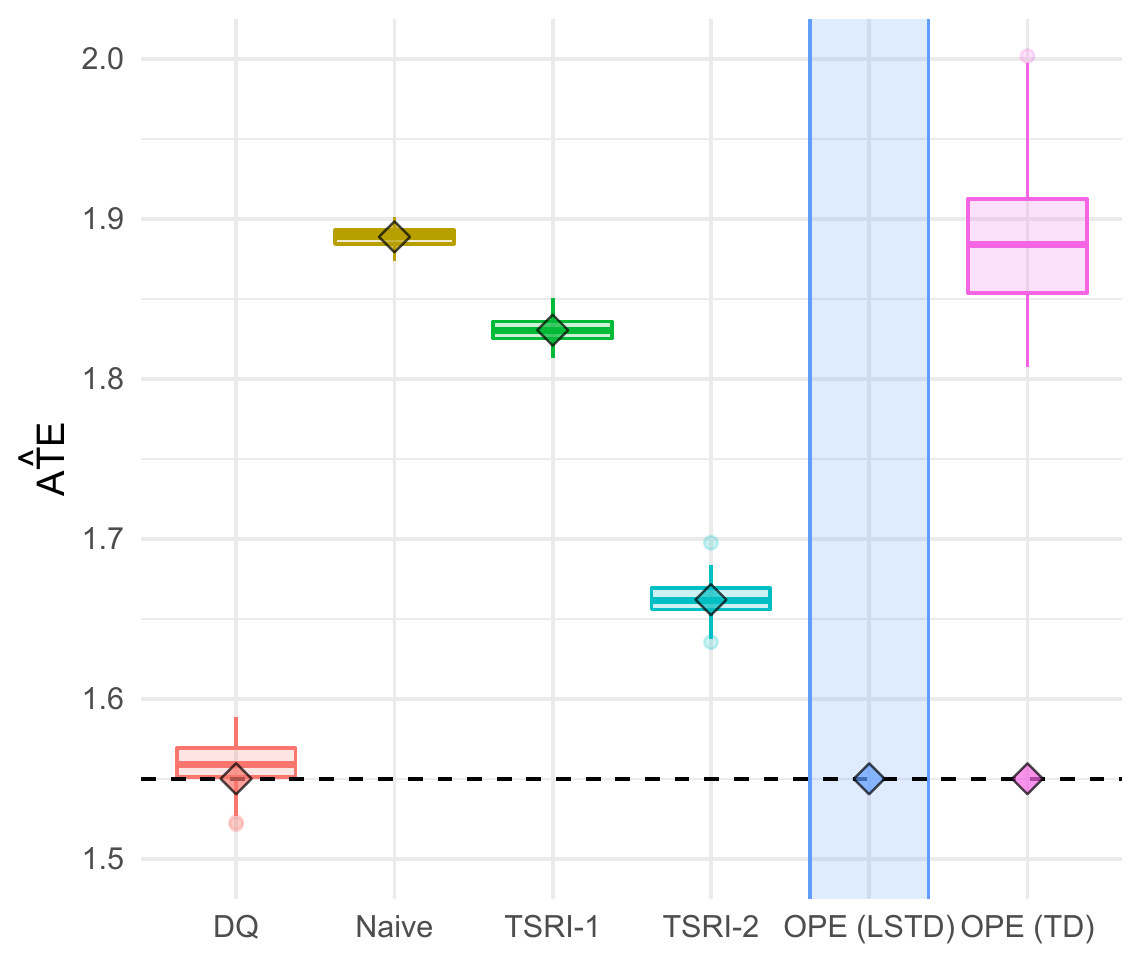}}
 \subfloat{\includegraphics[width=0.54\textwidth]{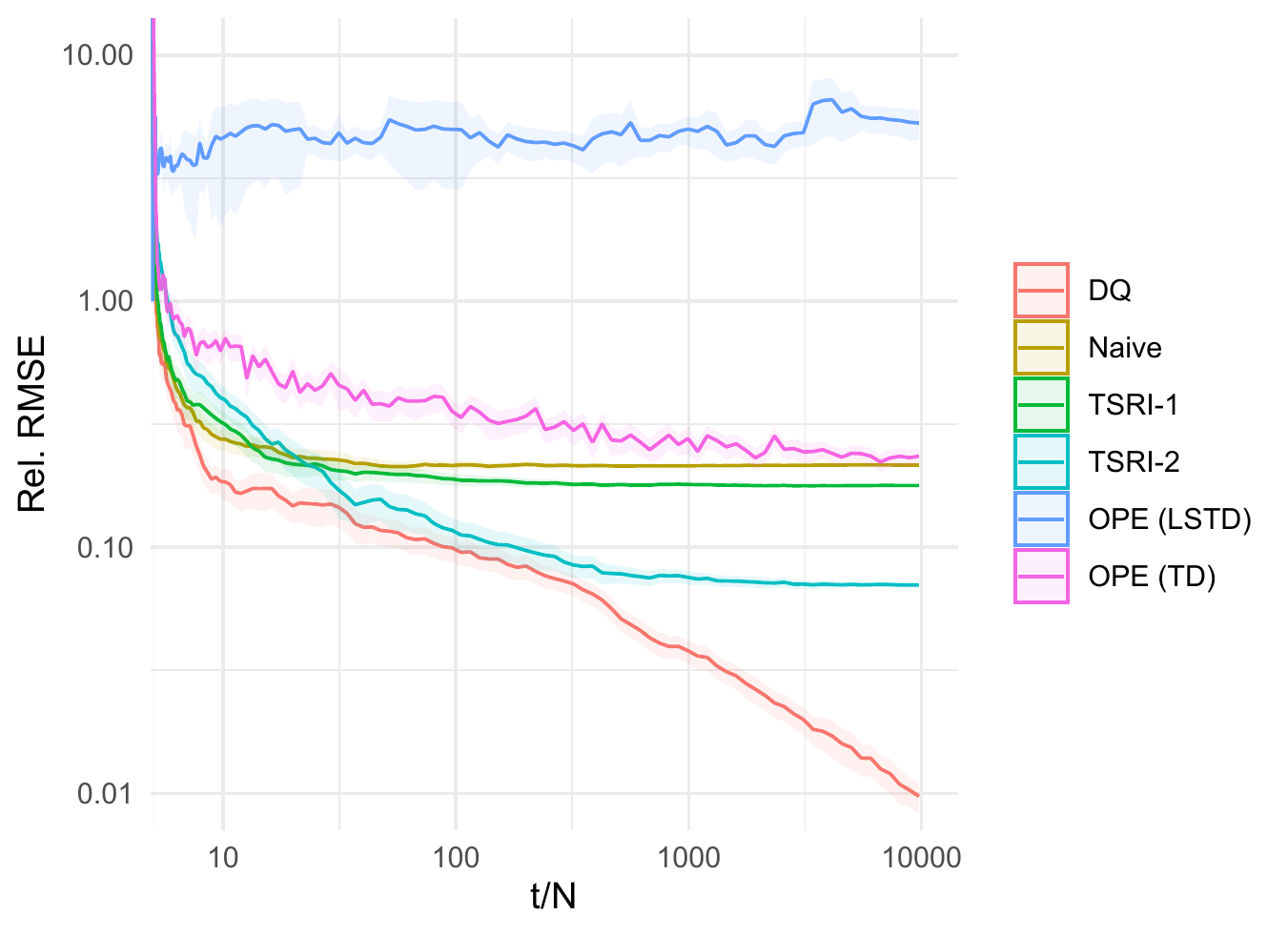}}
 \caption{Toy-example from \cite{johari2022experimental}. {\it Left}: Estimated ATE at time $t/N=10^{4}$ across $100$ trajectories. Dashed line indicates actual ATE. Diamonds indicate the asymptotic mean for each estimator. DQ shows compelling bias-variance tradeoff for this experimental budget.  {\it Right}: Relative RMSE vs. Time; DQ dominates the alternatives at all timescales. }
 \label{fig:birth-death}
\end{figure}
We run all estimators over 100 separate trajectories of length $t=10^{4} N$ of the above MDP initialized in its stationary distribution. Figure~\ref{fig:birth-death} summarizes the results of this experiment. Beginning with the left panel, which reports estimated quantities at $t = 10^4 N$, we immediately see:
\newline
\textbf{TSR improves on Naive: }The actual ATE in the experiment is $1.5\%$. Whereas it has the lowest variance of the estimators here, the Naive estimator has among the highest bias. The two TSR estimators reduce this bias substantially at a modest increase in variance. It is worth noting, as a sanity check, that these results precisely recreate those reported in \cite{johari2022experimental}. 
\newline
\textbf{OPE estimators are high variance: }The OPE estimators have the highest variance of those considered here. The TD estimator has the lower variance but this is simply because it is implicitly regularized. Run long enough, both estimators will recover the treatment effect.  
\newline
\textbf{DQ shows a compelling bias-variance tradeoff: }In contrast, the DQ estimator has the lowest bias at $t = 10^4 N$ and its variance is comparable to the TSR estimators (It is worth noting that run long enough, the DQ estimator had a bias of  $\sim-5 \times 10^{-7}$).
\newline
\textbf{Conclusions hold across experimental budgets: }Turning our attention briefly to the right chart in Figure~\ref{fig:birth-death}, we show the relative RMSE (i.e. RMSE normalized by the treatment effect) of the various estimators considered here {\em across all experimental budgets} $t$. RMSE effectively scalarizes bias and variance and we see that on this scalarization the DQ estimator dominates the other estimators considered here over all choice of $t$.

We note that specialized designs such as TSR can still be valuable in specific settings: when $\lambda \gg \mu$, for example, TSR is nearly unbiased (see \cite{johari2022experimental}), and can outperform DQ; see the appendix for such a study.
\subsection{A Large-Scale Ridesharing Simulator}
We next turn our attention to a city-scale ridesharing simulator similar to those used in production at large ride-hailing services. We will consider the problem of experimenting with changes to {\em dispatching} rules. Experimenting with these changes naturally creates Markovian interference by impacting the downstream supply/ positioning of drivers. Relative to the earlier toy example, the corresponding MDP here has an intractably large state-space, necessitating value function approximation for the DQ and OPE estimators.

\textbf{The Simulator: }Ridesharing admits a natural MDP; see e.g. \cite{qinReinforcementLearningRidesharing2021}. The state at the time of a request corresponds to that of all drivers at that time: position, assigned routes, riders, and the pickup/dropoff location of the request. Actions correspond to driver assignments and pricing decisions. The reward for a request is the price paid by the rider, less cost incurred to service the request.
%The problem of choosing a profit-maximizing dispatch policy, for example, can be modeled as an MDP where each time step corresponds to a request; the state is the set of drivers in the system, their assigned routes and riders, and the pickup / dropoff location of the request; actions correspond to dispatch decisions (i.e the assignment of a request to a driver) and 
%pricing, and the reward for each request is the price paid by the rider, less the cost incurred to service the request. 
Our simulator models Manhattan. Riders and drivers are generated according to real world data, based on \cite{TLCTripRecord}; this yields $\sim 300k$ requests and $\sim 7k$ unique drivers per real day. An arriving request is served a menu of options generated by a price engine. The rider chooses an option based on a choice model calibrated on taxi prices (for the outside option) and delay disutility. A dispatch engine assigns a driver to the rider; the engine chooses the driver who can serve the rider at minimal marginal cost, subject to the product's constraints. Finally drivers proceed along their assigned routes until the next request is received. The simulator implements pooling. Users can switch out demand and supply generation, pricing and dispatch algorithms, driver repositioning, and the choice model via a simple API. Other simulators exist in the literature \cite{qinReinforcementLearningRidesharing2021,yaoRidesharingSimulationPlatform2021}, but either lack an open-source implementation, or implement a subset of the functionality here.

\textbf{The Experiment: }We experiment with dispatch policies. Specifically, we consider assigning a request to an idle driver or a `pool' driver, i.e. a driver who already has riders in their car. A dispatch algorithm might prefer the former, but only if the cost of the resulting trip is at most $\alpha \%$ higher than the cost of assigning to a pool driver. We consider three experiments, each of which changes $\alpha$ from a baseline of $0$ to one of three distinct values: $30\%, 50\%$ or $70\%$, with ATEs of 0.5\%, -0.9\%, and -4.6\% respectively. As we noted earlier, we would expect significant interference in this experiment (or indeed any experiment that experiments with pricing or dispatch) since an intervention changes the availability / position of drivers for subsequent requests.

\begin{figure}[htbp]
  \centering
  \subfloat{\includegraphics[width=0.59\textwidth]{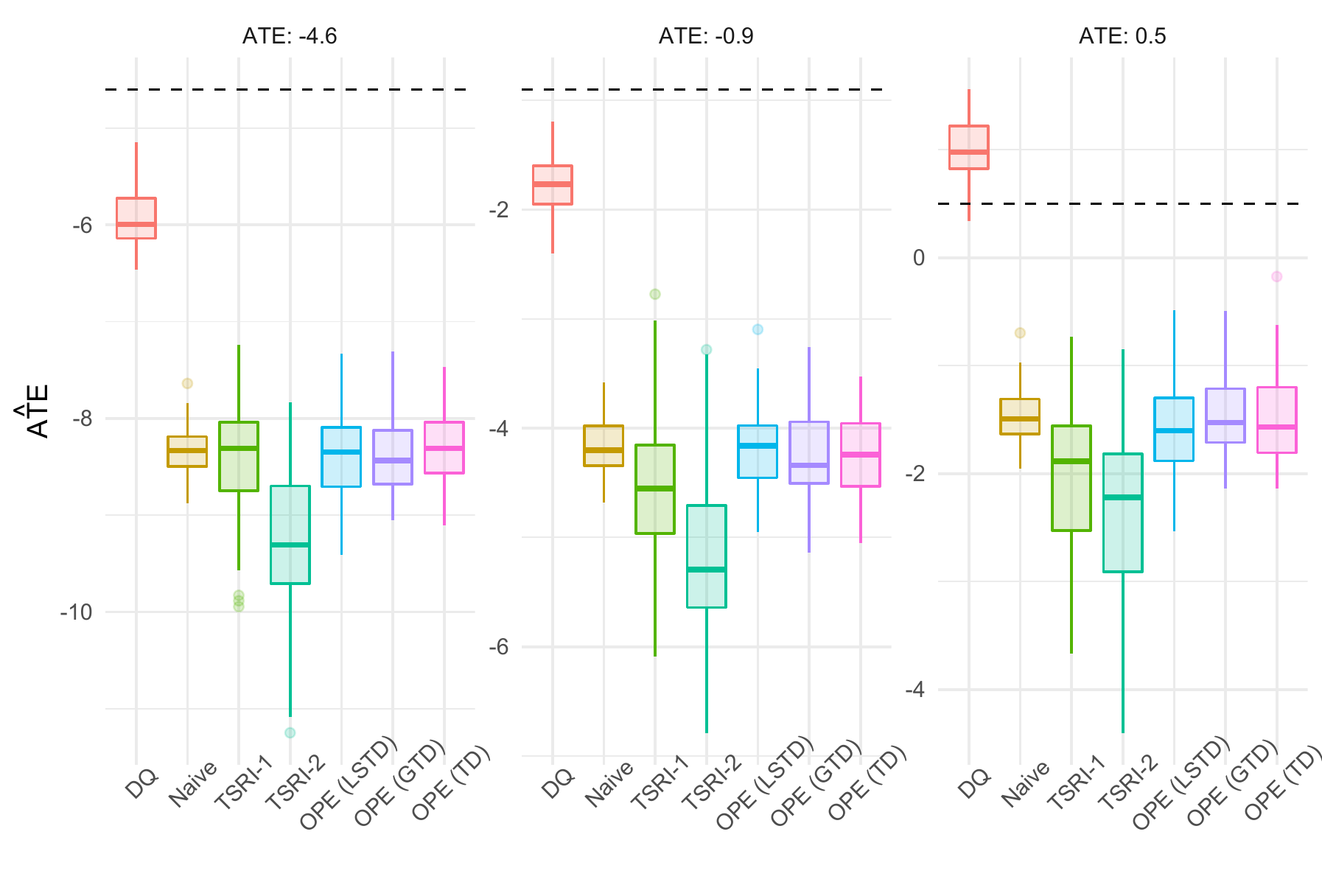}}
  \subfloat{\includegraphics[width=0.41\textwidth]{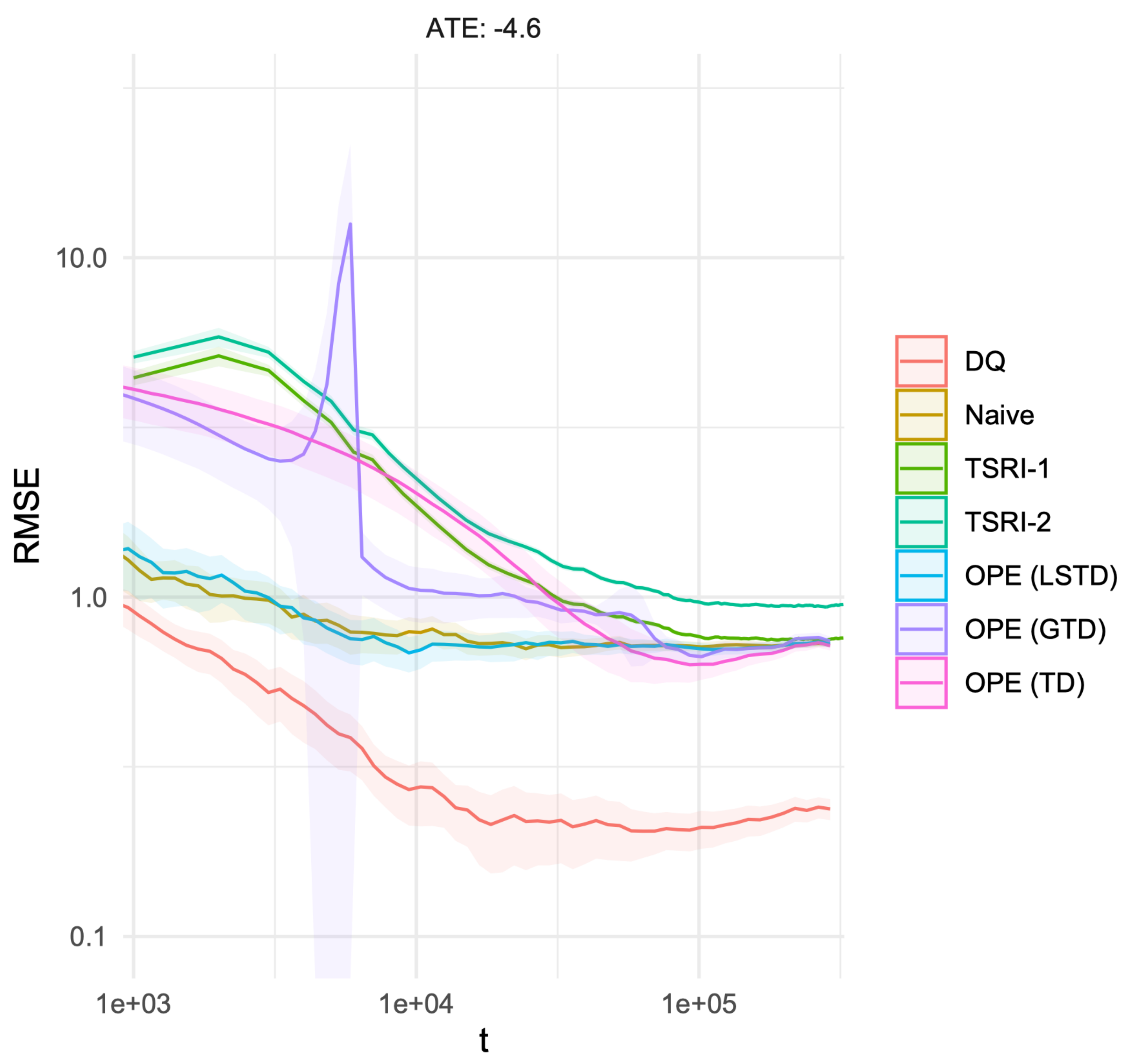}}
  \caption{Ridesharing model 
   {\it Left:} $\hat{\mathrm{ATE}}$ at $t = 3\times10^{5}$ over $50$ trajectories. Dashed line indicates actual ATE. DQ has lowest bias, and is only estimator to estimate correct sign of the treatment at all effect sizes.
  {\it Right:} RMSE vs. Time; DQ dominates at all time scales. 
  }
 \label{fig:rideshare}
\end{figure}

Figure~\ref{fig:rideshare} summarizes the results of the above experiments, wherein each estimator was run over $50$ independent simulator trajectories, each over $3\times10^5$ requests. The DQ and OPE estimators shared a common linear approximation architecture with basis functions that count the number of drivers at every occupancy level. We note that this approximation introduces its own bias which is not addressed by our theory. We immediately see: 
\newline
\textbf{Strong Impact of Interference: }As we might expect, interference has a significant impact here as witnessed by the large bias in the Naive estimator. 
\newline
\textbf{Incumbent estimators do not improve on Naive: }None of the incumbent estimators improve on Naive in this hard problem. This is also the case for the TSR designs, which in this large scale setting surprisingly appear to have significant variance. The OPE estimators have lower variance due to the regularization caused by value function approximation.
\newline
\textbf{DQ works: }In all three experiments, the bias in DQ (although in a relative sense higher than in the toy model) is {\em substantially} smaller than the alternatives, and also smaller than the ATE. This is evident in the left panel in Figure~\ref{fig:rideshare}. Notice that in the rightmost experiment (ATE $= 0.5$), DQ is the only estimator to learn that the ATE is positive. Like in the toy model, the right panel shows that these results are robust over experimentation budgets.  

\section{Discussion: Bias-Variance Tradeoffs, Policy Optimization} \label{sec:conclusion}

To summarize, we have shown that the DQ estimator achieves a surprising bias-variance tradeoff by applying on-policy estimation to the Markovian interference problem, and more generally to OPE. Here we draw further connections between the Naive, DQ, and OPE estimators and provide methods to realize other points on the bias-variance curve. Furthermore, we draw surprising connections between our estimator and trust-region methods in policy optimization, and show that DQ can serve as a drop-in replacement for these policy optimization surrogates with {\it  provably} lower bias.

\subsection{A $k^{\rm th}$-order Bias Correction}\label{sec:korder-correction}

As alluded to in Section~\ref{sec:estimator}, we can view the DQ estimator as a first-order correction to the Naive estimator, based on a Taylor series expansion of the ATE. This immediately motivates a $k^{\rm th}$-order correction, with the goal of obtaining estimators with bias $O(\delta^{k+1})$ for arbitrary $k$.

This correction turns out to have a suprising and intuitive form\footnote{For now, we assume for simplicity that rewards are only a function of state. Similar results can be derived for the more general case where $r$ is a function of both state and action, although the resulting formulas are more complex}. In short, to obtain the $k^{\rm th}$ order correction term for $k$ odd (the correction for $k$ even is 0), we simply compute the DQ estimator, but replace rewards in the MDP with the $(k-1)^{\rm th}$ order Difference-in-Q functions -- effectively a Difference-in-Qs-of-Difference-in-Qs.

Precisely, for some reward function $f: \mathcal{S} \mapsto \mathbb{R}$, we can define an auxiliary MDP with the same transition  probabilities, but rewards $f(s)$ at each state. Let $Q(s, a; f)$ be the corresponding Q function, under policy $\pi_{1 / 2}$. We now define
the $k^{\rm th}$ order Q-function to be $Q^{(k)}(s, a) = Q(s, a; f^{(k-1)})$, where the rewards are defined as $f^{(k-1)}(s) = \frac{1}{2} \left( Q^{(k-1)}(s, 1) - Q^{(k-1)}(s, 0) \right)$; in other words, the previous-order Difference-in-Q functions. We take as a base case $f^{(0)}(s) = r(s)$. Finally, we can define the $K^{\rm th}$-order Difference-in-Qs estimate of the ATE to be the sum of all lower-order correction terms: $
\hat{\mathrm{ATE}}_{\rm DQ}^{(K)} = \sum_{k \text{ odd}, k \leq K} \mathbb{E}_{\rho_{1/2}}\left[Q^{(k)}(s, 1) - Q^{(k)}(s, 0)\right]$.

In principle this approach enables {\it off-policy}  evaluation with arbitarily low bias -- entirely via estimation of {\it on-policy} quantities. One can verify that ${\rm{ATE}}_{\rm DQ}^{(1)}$ is the expected value of the DQ estimator. We now generalize \cref{th:dynkin_bias} to provide a bias bound for the $k^{\rm th}$ order correction\footnote{The variance of such plug-in estimators can be bounded by iteratively applying \cref{lem:non-expansive-lemma}, which is omitted for simplicity.}:

\begin{theorem}\label{th:korder-bias}
For any $K=0, 1, 2, \dotsc$, we have $\left|{\rm{ATE}} - {\rm{ATE}}_{\rm DQ}^{(K)}\right| \leq C' \left(\frac{1}{1-\lambda}\right)^{K+1} \delta^{K+1} r_{\max}$, where $C'$ is a constant depending (polynomially) on $\log(C).$
\end{theorem}

\subsection{Interpolating from OPE to Naive via Regularization}

Here, we view the DQ estimator again as an intermediate point on the bias-variance curve between Naive and OPE. This time, however, we interpolate between these extremes by regularizing certain key nuisance parameters in estimating the ATE.

\paragraph{An OPE meta-estimator.} First, we situate the DQ estimator in the context of existing OPE techniques. Consider the following exact identity for the ATE:  $\mathrm{ATE} = \E_{\rho_{1 / 2}}[\zeta(s) (Q_{\pi_{1 / 2}}(s, 1) - Q_{\pi_{1 / 2}}(s, 0))]$ where $\zeta(s) = \frac{1}{2}\frac{ \rho_{1}(s) + \rho_{0}(s)}{\rho_{1 / 2}(s)}$ is the likelihood ratio of the stationary distributions. A variety of OPE estimators -- including doubly-robust (\cite{kallus2020double,ueharaMinimaxWeightQFunction2020}) and primal-dual (\cite{daiBoostingActorDual2017,tangDoublyRobustBias2019}) estimators -- in fact estimate ATE explicitly by plugging in estimates $\hat{\zeta}, \hat{Q}_{\pi_{1 / 2}}$ of the likelihood ratio and value functions (referred to as the ``doubly-robust meta-estimator'' in \cite{kallus2020double}):

\begin{equation}
  \label{eq:dr-meta}
 \hat{\mathrm{ATE}}_{\rm DR} = \frac{1}{|T_{1}|}\sum_{t \in T_1}\hat{\zeta}(s_t) \hat{Q}_{\pi_{1 / 2}}(s_t, 1) - \frac{1}{|T_0|} \sum_{t\in T_0}\hat{\zeta}(s_t) \hat{Q}_{\pi_{1 / 2}}(s_t, 0)
\end{equation}
%\paragraph{Refining the Bias-Variance Tradeoff.} Immediately, we see that that by taking the likelihood ratio to be a constant $\hat{\zeta}(s) = 1 \; \forall s$, we recover the DQ estimator $\hat{\mathrm{ATE}}_{\rm DQ}$. Furthermore, if we then take  $\hat{V}_{\pi_{1 / 2}}$ to be any constant $\hat{V}_{\pi_{1 / 2}}(s) = c \; \forall s$, we recover the Naive estimator\footnote{To see this, observe that $\hat{Q}_{\pi_{1 / 2}}(s, a) = r(s, a) + \E \left[\hat{V}_{\pi_{1 / 2}}(s') | s, a\right] = r(s, a) + c$}. The DQ and Naive estimators' relationship to OPE then becomes clear: we obtain DQ by choosing a minimal variance (but highly biased) estimator of $\zeta$; and we obtain the Naive estimator by subsequently choosing a minimal variance (but highly biased) estimator of $V$. This suggests that we can interpolate between these extremes by making more refined bias-variance tradeoffs in estimating $\hat{\zeta}$ and $V_{\pi_{1 / 2}}$. It turns out that several natural approaches to variance reduction provide exactly such an interpolation.

\noindent\textbf{\textsf{Explicit regularization}.} In estimating $\hat{\zeta}(s)$, one can directly penalize its deviation from one, where increasing the penalty interpolates from OPE to DQ. Given that estimation of $\hat{\zeta}(s)$ is the key difference between DQ and unbiased OPE -- and therefore the source of the massive variance gap (Theorems~\ref{th:dynkin_var} and \ref{th:cramer-rao}) -- we would expect this to be a particularly powerful approach to OPE, and indeed similar penalties have produced strong empirical performance \cite{nachum2019dualdice}. Similarly, one can directly penalize the deviation of $\hat{V}_{\pi_{1 / 2}}$ from zero, as in regularized variants of LSTD (see e.g. \cite{kolterRegularizationFeatureSelection2009}). As we increase the regularization penalty on $\hat{\zeta}(s)$, we interpolate from OPE to DQ; additionally increasing the regularization penalty on $\hat{V}_{\pi_{1 / 2}}$ then interpolates from DQ to Naive. Approaches combining both forms of regularization have been explored in \cite{yangOffPolicyEvaluationRegularized2020}.

\noindent\textbf{\textsf{Function approximation}.} More generally, one can restrict $\hat{\zeta}(s)$ and $\hat{V}_{\pi_{1 / 2}}$ to lie in particular function classes, with one extreme being any mapping $\mathcal{S} \mapsto \mathbb{R}$, and the other extreme being the constant functions $\hat{V}_{\pi_{1 /2}}(s) = c$ or $\hat{\zeta}(s) = 1$. As one example, when the state space is massive we may approximate it using state aggregation. At the extreme, aggregating all states into a single aggregate state implies that the value function (or likelihood ratio) must be a constant. As the aggregation for $\hat{\zeta}(s)$ goes from fine to coarse, we  interpolate between OPE and DQ; increasing the coarseness of $\hat{V}_{\pi_{1 / 2}}(s)$ then interpolates between DQ and Naive.
%
%{\bf Discounting.} A common technique to estimate the average reward value function is to instead estimate a  discounted reward value function $Q_{\gamma}(s, a) = \mathbb{E}\left[\sum_{t=0}^T \gamma^{t} r(s_t, a_t) | s_{0} = s, a_0 = a\right]$, motivated by the fact that we obtain exactly the average-reward value function $Q$ as the discount rate $\gamma$ goes to one (under the proper scaling \cite{putermanMarkovDecisionProcesses2014}). Implementing DQ with $\hat{Q}_{\pi_{1 / 2}}(s, a) = (1 - \gamma)Q_{\gamma}(s,a)$ yields the exact DQ estimator as $\gamma \to 1$, and the Naive estimator as $\gamma \to 0$.

\subsection{DQ as a Policy Optimization Objective}
\label{sec:policy-optimization}

\paragraph{Surrogate objectives in trust-region methods}
DQ also has a suprising relationship to trust-region methods \cite{schulman2015trust,schulman2017proximal,kakade2002approximately}. At each iteration, these methods essentially solve an {\it offline} policy optimization problem: using data collected under some ``behavioral''policy $\pi_{\rm b}$, they evaluate (and subsequently optimize) a candidate policy $\pi$. The policy evaluation step is exactly an OPE problem, and they construct a surrogate objective based on an identity sometimes referred to as Dynkin's identity \cite{dynkin1965markov}: $\lambda({\pi}) = \E_{\pi_{b}}[r(s, a)] + \E_{s \sim \rho_{\pi_{b}}, a \sim \pi}\left[ \frac{\rho_{\pi}(s)}{\rho_{\pi_{b}}(s)} \left( Q^{\pi_{b}}(s, a) - V^{\pi_{b}}(s) \right) \right]$, where $\rho_{\pi_b}, V^{\pi_b}, Q^{\pi_b}$ is the stationary distribution, $V$-function, and $Q$-function of the policy $\pi_{b}$ respectively and $\rho_{\pi}$ is the stationary distribution of $\pi$. The referenced trust-region methods then effectively take the likelihood ratio $\rho_{\pi}(s) / \rho_{\pi_{b}}(s)$ to be identically one, yielding the (idealized) surrogate objective:
$\hat{\lambda}_{\rm TR}({\pi}) = \E_{\pi_{b}}[r(s, a)] + \E_{s \sim \rho_{\pi_{b}}, a \sim \pi}\left[  Q^{\pi_{b}}(s, a) - V^{\pi_{b}}(s) \right]$.

\paragraph{A DQ-based surrogate} As it turns out, DQ derives from a very similar identity -- with a small but critical difference which allows DQ to obtain \textit{even lower bias}. Applying the peturbation bound in \cref{lem:perturbation-stationary-distribution} to $\lambda(\pi)$, we obtain a slightly different identity for $\lambda(\pi)$: $\lambda({\pi}) = \E_{\pi_{b}}[r_{\pi}(s)] + \E_{s \sim \rho_{\pi_{b}}, a \sim \pi}\left[ \frac{\rho_{\pi}(s)}{\rho_{\pi_{b}}(s)} \left( Q^{\pi_{b}}(s, a; r_{\pi}) - V^{\pi_{b}}(s; r_{\pi})\right) \right]$ where $r_{\pi}(s) = \mathbb{E}_{a \sim \pi}\left[r(s, a)\right]$ is the expected reward under $\pi$, and recall that $Q^{\pi_{b}}(\cdot, \cdot; r_{\pi}), V^{\pi_{b}}(\cdot, \cdot; r_{\pi})$ are the value functions for an auxiliary MDP with the same transition probabilities, but {\it taking $r_{\pi}$ to be the reward function}. The biased  form of this estimator, which forms the basis of the DQ estimator\footnote{The DQ estimator in our original setting can actually be derived as either $\hat{\lambda}_{\rm TR}(\pi_{1}) -\hat{\lambda}_{\rm TR}(\pi_{0}) $ or $\hat{\lambda}_{\rm DQ}(\pi_{1}) -\hat{\lambda}_{\rm DQ}(\pi_{0}) $, but this results from a very suprising cancellation of terms in the subtraction; i.e. $\hat{\lambda}_{\rm TR}(\pi) $ and $\hat{\lambda}_{\rm DQ}(\pi)$ are individually different estimators with very different properties, as we will see; and the higher-order corrections must be derived from $\hat{\lambda}_{\rm DQ}$.}, is: $\hat{\lambda}_{\rm DQ}({\pi}) = \E_{\pi_{b}}[r_{\pi}(s)] + \E_{s \sim \rho_{\pi_{b}}, a \sim \pi}\left[   Q^{\pi_{b}}(s, a; r_{\pi}) - V^{\pi_{b}}(s; r_{\pi}) \right]$
which is precisely $\hat{\lambda}_{\rm TR}(\pi)$, but computed on an MDP with rewards $r_{\pi}$.
\paragraph{Lower-order bias} This striking resemblance between the surrogatees $\hat{\lambda}_{\rm DQ}(\pi)$ and $\hat{\lambda}_{\rm TR}(\pi)$ naturally raises the question of how they compare. As it turns out, the simple act of replacing the rewards with $r_{\pi}$ in $\hat{\lambda}_{\rm DQ}$ has significant consequences in terms of bias:

\begin{theorem}[Bias of the DQ surrogate] Suppose it holds that $d_{\rm{TV}}(p(s, a, \cdot), p(s, a', \cdot)) \leq \delta $  for all $s, a, a'$, and that $d_{\rm{TV}}(\pi(s, \cdot), \pi'(s, \cdot)) \leq \delta'$ for all $s$. Then, the biases of $\hat{\lambda}_{\rm TR}(\pi)$\footnote{This result for $\hat{\lambda}_{\rm TR}$ is in fact a slightly refined version of the key perturbation bound in \cite{schulman2015trust}.} and $\hat{\lambda}_{\rm DQ}(\pi)$ satisfy
  \begin{align*}
    |\hat{\lambda}_{\rm TR}(\pi) - \lambda^\pi| = O(\delta (\delta')^2 ) && |\hat{\lambda}_{\rm DQ}(\pi) - \lambda^\pi| = O((\delta\delta')^2 )
  \end{align*}
\end{theorem}

This characterization is sharp, in that there exist non-pathological examples where the exact bias of $\hat{\lambda}_{\rm DQ}$ is a factor $\delta$ smaller than that of $\hat{\lambda}_{\rm TR}$. Crucially, this means that even if the distance between {\it policies} has no non-vacuous upper bound (i.e. $\delta' = 2$), as long as the resulting {\it transition functions} are similar, then the bias of $\hat{\lambda}_{\rm DQ}$ will be small, whereas the bias of $\hat{\lambda}_{\rm TR}$ can be of the order of $\delta$. This immediately suggests that optimizing $\hat{\lambda}_{\rm DQ}$ with respect to $\pi$ should allow for both larger and more accurate policy improvement steps.

\section{Conclusion} \label{sec:real-conclusion}
We propose a novel estimator, the DQ estimator, to solve the interference problem in experiments with simple randomized designs. The DQ estimator achieves second-order bias in estimating the average treatment effect, while its variance can be exponentially smaller than that of any unbiased estimator. We conducted a large scale ride-hailing experiment that demonstrated the superior performance of the DQ estimator over state-of-the-art approaches. The striking and rigorous bias-variance trade-offs induced by the DQ estimator and its generalizations provide a new lens for general off-policy evaluation and policy optimization in reinforcement learning.

\linespread{.5}
\begin{multicols}{2}
\bibliographystyle{abbrv}
%\singlespacing
\footnotesize
%\scriptsize
\bibliography{reference} % if more than one, comma separated
%
% CASE 2: BiBTeX used to generate mypaper.bbl (to be further fine tuned)
%\input{mypaper.bbl} % outcomment this line in Case 2
%
%If you don't use BiBTex, you can manually itemize references as shown below.
\end{multicols}

%
%\bibliographystyle{unsrt}% Style BST file (imsart-number.bst or imsart-nameyear.bst)
%\bibliography{reference.bib}       % Bibliography file (usually '*.bib')

%%%%%%%%%%%%%%%%%%%%%%%%%%%%%%%%%%%%%%%%%%%%%%%%%%%%%%%%%%%%

%%%%%%%%%%%%%%%%%%%%%%%%%%%%%%%%%%%%%%%%%%%%%%%%%%%%%%%%%%%%
%\section*{Checklist}

%%% BEGIN INSTRUCTIONS %%%
%The checklist follows the references.  Please
%read the checklist guidelines carefully for information on how to answer these
%questions.  For each question, change the default \answerTODO{} to \answerYes{},
%\answerNo{}, or \answerNA{}.  You are strongly encouraged to include a {\bf
%justification to your answer}, either by referencing the appropriate section of
%your paper or providing a brief inline description.  For example:
%\begin{itemize}
%  \item Did you include the license to the code and datasets? \answerYes{See Section~\ref{gen_inst}.}
%  \item Did you include the license to the code and datasets? \answerNo{The code and the data are proprietary.}
%  \item Did you include the license to the code and datasets? \answerNA{}
%\end{itemize}
%Please do not modify the questions and only use the provided macros for your
%answers.  Note that the Checklist section does not count towards the page
%limit.  In your paper, please delete this instructions block and only keep the
%Checklist section heading above along with the questions/answers below.
%%% END INSTRUCTIONS %%%

\newpage
\appendix

\section*{Appendix}
\section{Notation}
For a vector $a \in \R^{n}$, we use $\norm{a}_{1} = \sum_{i=1}^{n} |a_i|$ and $\norm{a}_{\infty} = \max_{i=1}^{n} |a_i|.$ For a matrix $M \in \R^{n\times m}$, we use $\norm{M}_{1,\infty} =\max_{1\leq i\leq n}\sum_{j=1}^{m} |a_{ij}|$ to represent the maximal row-wise $l_1$-norms. We use $\one$ to represent the vectors with all ones. We use $A^{\#}$ to represent the group inverse of $A$. For an irreducible and aperiodic Markov chain with associated transition matrix $P$ and the stationary distribution $\rho$, there is $(I-P)^{\#} = (I-P+\one\rho^{\top})^{-1} - \one\rho^{\top}.$

\section{Analysis of the Example}\label{sec:calculation-example}
To begin, let us derive the ATE. Under policy $\pi_0$, the transition matrix is 
$$
P_{0} = \begin{bmatrix}(1-p)\lambda+\mu & p\lambda\\ \mu & \lambda\end{bmatrix}
$$
and the stationary distribution is $\rho_{0} = [\frac{\mu}{\mu + \lambda p}, \frac{\lambda p}{\mu + \lambda p}]^{\top}$ accordingly. Similarly, one can verify under policy $\pi_1$, the transition matrix is 
$$
P_{1} =  \begin{bmatrix}(1-p-\delta)\lambda+\mu & (p+\delta)\lambda\\ \mu & \lambda\end{bmatrix}
$$
and the stationary distribution is $\rho_{1} =  [\frac{\mu}{\mu + \lambda (p+\delta)}, \frac{\lambda (p+\delta)}{\mu + \lambda (p+\delta)}]^{\top}.$ Let $r_0 = [\lambda p, 0]^{\top}, r_1 = [\lambda (p+\delta), 0]^{\top}$ be the reward vector under actions $0$ or $1$. Then, the ATE is 
\begin{align*}
{\rm{ATE}} 
&= r_1^{\top} \rho_1 - r_0^{\top} \rho_0 \\
&= \frac{\mu \lambda (p+\delta)}{\mu + \lambda (p+\delta)} - \frac{\mu \lambda p}{\mu + \lambda p}\\
&= \frac{\delta \mu^2 \lambda}{(\mu + \lambda (p+\delta)) (\mu + \lambda p)}.
\end{align*}

Consider the transition matrix for $\pi_{1/2}$,
\begin{align*}
P = \begin{bmatrix}(1-p-\delta/2)\lambda+\mu & (p+\delta/2)\lambda\\ \mu & \lambda\end{bmatrix}.
\end{align*}
Then one can verify that the stationary distribution $\rho_{1/2}$ is
\begin{align*}
\rho_{1/2} =  \left[\frac{\mu}{\mu + \lambda (p+\delta/2)}, \frac{\lambda (p+\delta/2)}{\mu + \lambda (p+\delta/2)}\right]^{\top}.
\end{align*}
The naive estimator is 
\begin{align*}
\E[\hat{{\rm ATE}}_{\rm{NV}}] = \frac{\delta \lambda \mu}{\mu + \lambda (p+\delta / 2)}.
\end{align*}

Next, we consider the computation of $\E_{\rho_{1/2}}[\hat{{\rm{ATE}}}_{\rm DQ}]$, which can be written as
\begin{align*}
\E_{\rho_{1/2}}[\hat{{\rm{ATE}}}_{\rm DQ}] = \rho_{1/2}^{\top} (Q_1 - Q_0)
\end{align*}
where $Q_a$ is the Q-value vector for the policy $\pi_{1/2}$ under the action $a$. Furthermore, consider the following Bellman equation for $Q$-value function:
\begin{align*}
Q(s, a) = r(s, a) - \lambda^{1/2} + \sum_{s', a'} P_{a}(s, s') \frac{1}{2} Q(s', a').
\end{align*}
One can verify that one solution of the above equations is
\begin{align*}
Q(0, 0) = \frac{\mu \lambda p}{\mu + \lambda p}, &\quad Q(0, 1) = 0\\
Q(1, 0) = \frac{\mu \lambda (p+\delta)}{\mu + \lambda p}, &\quad Q(1,1) = 0
\end{align*}
Therefore,
\begin{align*}
\E_{\rho_{1/2}}[\hat{{\rm{ATE}}}_{\rm DQ}]
&=  \frac{\mu}{\mu + \lambda (p+\delta/2)} (Q(0, 1) - Q(0, 0))\\
&=  \frac{\mu}{\mu + \lambda (p+\delta/2)}  \frac{\mu \lambda \delta}{\mu + \lambda p}.
\end{align*}

For the bias induced by the DQ estimator, we have
\begin{align*}
{\rm{ATE}} - \E_{\rho_{1/2}}[\hat{{\rm{ATE}}}_{\rm DQ}]
&= \frac{\delta \mu^2 \lambda}{(\mu + \lambda (p+\delta)) (\mu + \lambda p)} - \frac{\mu}{\mu + \lambda (p+\delta/2)}  \frac{\mu \lambda \delta}{\mu + \lambda p}\\
&= \delta \frac{\mu^2 \lambda}{\mu + \lambda p} \left(\frac{1}{\mu + \lambda (p+\delta)} - \frac{1}{\mu + \lambda (p+\delta/2)} \right)\\
&\approx \frac{\delta}{2} \frac{\lambda}{(\mu+\lambda p)} {\rm{ATE}}
\end{align*}
where $\approx$ ignore the terms $O(\delta^3).$ This completes the analysis.

\section{Proof of Theorem \ref{th:dynkin_var}}\label{sec:proof-variance-bound}

\subsection{Entry-wise Non-Expansive Lemma}
To begin, we present a lemma that is the key enabler of establishing the striking variance improvement of the DQ estimator  over the un-biased estimator. The lemma simply states
\begin{replemma}{lem:non-expansive-lemma}[Entry-wise non-expansive lemma]
Let $W: \R^{|\Sscr|} \rightarrow \R^{|\Sscr|}$ be a map denoted by $W(\rho) := (I-P_{1/2})^{\#\top}(P_1-P_0)^{\top} \rho.$ Then, for any $s \in \Sscr$, 
\begin{align*}
\frac{1}{c}\left|W(\rho_{1/2}) (s)\right| \leq \rho_{1/2}(s)
\end{align*}
where $c := 4\frac{\ln(C) +  \ln \left(1/\rho_{\min}\right) + 1}{1-\lambda}.$
\end{replemma}
This is to say, the mapping $\frac{1}{c} W$ does not expand $\rho_{1/2}$ in terms of entry-wise values. To see the necessity of this lemma and gain some intuition, consider a special case where $P_{0}, P_{1}, P_{1/2}, \rho_0, \rho_1, \rho_{1/2}$ are all known, while only the rewards are unknown and can only be sampled under the distribution $\rho_{1/2}.$ For simplicity, assume $r_0=r_1=r$ and the sample for $r_{t} = r(s_t) + \epsilon_{t}$ is i.i.d from $\rho_{1/2}$ with some exogenous noise $\epsilon_{t} \sim \mathcal{N}(0, 1).$ Let us denote the empirical average estimator for $r$ be $\hat{r}.$ 
By CLT, we have 
$$
\sqrt{T}(\hat{r}-r) \overset{d}{\rightarrow} \mathcal{N}(0, D^{-1})
$$
where $D$ is a diagonal matrix with entries $D_{s,s} = \rho_{1/2}(s).$ This limiting variance captures the intuition that, for the state that is rarely visited, the variance for $\hat{r}(s) - r(s)$ can blow up. In fact, consider an un-biased estimator $(\rho_{1}^{\top}-\rho_{0}^{\top})\hat{r}$ for the ATE, $(\rho_1^{\top}-\rho_0^{\top})r$, (this is the un-biased estimator that achieves the optimal variance), we have
\begin{align*}
\sqrt{T} \left((\rho_{1}^{\top}-\rho_{0}^{\top})\hat{r} - {\rm{ATE}}\right)\overset{d}{\rightarrow} \mathcal{N}(0, \sigma_{0}^2)
\end{align*}
where
\begin{align*}
\sigma_0^2 
&:= (\rho_{1}^{\top}-\rho_{0}^{\top}) D^{-1} (\rho_{1}-\rho_{0})\\
&= \sum_{s} \frac{(\rho_1(s) - \rho_0(s))^2}{\rho_{1/2}(s)}.
\end{align*}
Note that there is no guarantee for the likelihood ratio $\frac{\rho_0(s)}{\rho(s)}$ and $\frac{\rho_1(s)}{\rho(s)}$ and in general $\sigma_0^2 = \Omega\left(\frac{1}{\rho_{\min}}\right)$ and this is the price to pay for the un-biased off-policy evaluation. 

On the other hand, one can verify that the DQ estimator is simply
\begin{align*}
{\rm{ATE}}_{D} = \rho_{1/2}^{\top} \left(P_0-P_1\right) (I-P_{1/2})^{\#} \hat{r}.
\end{align*}
This leads to the limiting variance of ${\rm{ATE}}_{\rm DQ}$:
\begin{align*}
\sqrt{T}\left({\rm{ATE}}_{D} - \E_{\rho_{1/2}}[{\rm{ATE}}_{D}]\right) \overset{d}{\rightarrow} \mathcal{N}(0, \sigma_{1}^2)
\end{align*} 
where 
\begin{align*}
\sigma_{1}^2 &:= \rho_{1/2}^{\top} \left(P_0-P_1\right) (I-P_{1/2})^{\#} D^{-1} (\rho_{1/2}^{\top} \left(P_0-P_1\right) (I-P_{1/2})^{\#})^{\top} \\
&=  \norm{\rho_{1/2}^{\top} \left(P_0-P_1\right) (I-P_{1/2})^{\#} D^{-1/2}}^2.
\end{align*}
By the definition of $W$ mapping, we then have
\begin{align*}
\sigma_{1}^2 
&= \sum_{s} \left(W(\rho_{1/2})(s) \frac{1}{\rho_{1/2}(s)^{1/2}}\right)^2\\
&\overset{(i)}{\leq} \sum_{s} \frac{1}{c^2} \rho_{1/2}(s)\\
&=  \frac{1}{c^2}.
\end{align*}
where (i) is due to \cref{lem:non-expansive-lemma}. Then $\sigma_{1}$ is in the order of $\log(1/\rho_{\min})$. In fact, without \cref{lem:non-expansive-lemma}, a loose analysis will provide $\sigma_{1}^2 = \Omega(1/\rho_{\min})$, which is in the same order of $\sigma_0^2$, that shows no advantage of using DQ estimator. Essentially \cref{lem:non-expansive-lemma} characterizes the explicit superiority of evaluating on-policy quantities over off-policy quantities. We believe this novel lemma is of independent interest for the field of OPE. The proof is postponed to the end of the section.

\subsection{Outline of the Proof}
In this section, we present the outline of the proof for \cref{th:dynkin_var}. We aim to use Markov chain CLT (\cite{jones2004markov}) to provide the asymptotic normality of our estimator. Note that Markov chain CLT states that for a Markov chain $X_{1}, X_{2}, \dotsc, $ and a bounded function $u$ with the domain on the state space, there exists $\Sigma_{u}$ such that
\begin{align*}
\sqrt{T}\left(\frac{1}{T}\sum_{t=1}^{T} u(X_t) - u^{*}\right) \overset{d}{\rightarrow} N(0, \Sigma_{u})
\end{align*}
where $u^{*}$ is the expected value of $u$ under the stationary distribution of the Markov chain. 

\textbf{Delta method.} Unfortunately, the estimator $\tauDK$ can not be directly written as an empirical average of some function $u.$ To address this issue, we use ``delta method'' (traced back to \cite{doob1935limiting}, see \cref{lem:mapping-CLT}). In particular, we write $\tauDK = f(u_T)$ as a function of a random vector $u_{T}$ given by $u_{T} := \frac{1}{T}\sum_{t=1}^{T} u(X_{t}).$ Under some minor conditions, ``delta method'' states that 
\begin{align*}
\sqrt{T} \left(f(u_{T}) - f(u^{*})\right) \overset{d}{\rightarrow} N(0, \sigma_{f}^2)
\end{align*}
where $\sigma_{f}^2 := \nabla f(u^{*})^{\top} \Sigma_{u} \nabla f(u^{*})$ and $\nabla f(u^{*})$ is the gradient of $f$ evaluating at the point $u^{*}.$ This forms the basis of proving \cref{th:dynkin_var}. 

\textbf{Linearization.} To simplify the analysis for $\sigma_{f}$, instead of computing $\Sigma_{u}$ explicitly, we ``linearize'' the function $f$ by defining $\tilde{f}(X_{t}) := \nabla f(u^{*})^{\top}(u(X_t)-u^{*})$ and the delta method in fact implies (see \cref{lem:linearization})
\begin{align*}
\sqrt{T}\left(\frac{1}{T}\sum_{t=1}^{T} \tilde{f}(X_t)\right) \overset{d}{\rightarrow} N(0, \sigma_{f}^2),
\end{align*}
i.e., the linearized $f$ converges with the same limiting variance as the original $f.$ Therefore, we can focus on $\tilde{f}$ for analyzing $\sigma_{f}.$ 

\textbf{Bounding $\sigma_f$ with Entry-wise Non-expansive Lemma.} To bound $\sigma_{f}$, we will invoke \cref{lem:bound-for-sigma}, which states that
\begin{align*}
\sigma_{f} \leq \sqrt{2}\sqrt{\frac{2\ln(C)+1}{1-\lambda}}\tilde{f}_{\max}
\end{align*}
where $\tilde{f}_{\max} := \max_{s} |\tilde{f}(s)|.$ Then the problem boils down to bound $\tilde{f}_{\max}$, which will be controlled by \cref{lem:non-expansive-lemma}.

Next, we present the proof in full details. 
\subsection{Delta Method and Linearization}
To begin, consider the Markov chain $X_{t} = (s_{t}, a_{t}, s_{t+1}).$ For $a \in \{0, 1\}$, denote $F^{(a)}, h^{(a)}$ by
\begin{align} \label{eq:def-F-h}
F^{(a)}(X_t) &:= 2E_{s_t}E_{s_{t+1}}^{\top}\cdot 1(a_{t}=a)\\
h^{(a)}(X_t) &:= 2r(s_t, a_t) \cdot E_{s_t}\cdot 1(a_t=a)
 \end{align}
 where $E_{s}$ is a vector with all entries zero except that the $s$-th entry is one. Let $F^{(a)}_{T} \in \R^{|\Sscr| \times |\Sscr|}, h^{(a)}_{T} \in \R^{|\Sscr|}$ be the empirical average of the function $F^{(a)}$ and $h^{(a)}$:
\begin{align*}
F^{(a)}_{T} &:= \frac{1}{T}\sum_{t=1}^{T} F^{(a)}(X_{t})\\
h^{(a)}_{T} &= \frac{1}{T}\sum_{t=1}^{T} h^{(a)}(X_{t}).
\end{align*} 
We aim to write $\tauDK := f(F^{(0)}_T, F^{(1)}_T, h^{(0)}_T, h^{(1)}_T)$ as a function of $F^{(0)}_T, F^{(1)}_T, h^{(0)}_T, h^{(1)}_T$ for applying delta method. To do so, let $D_{T}^{(a)}$ be an diagonal matrix with entries $D_{T}^{(a)}(s, s) = \sum_{s'} F_{T}^{(a)}(s, s').$ One can verify that
\begin{align*}
\hat{V} = (D_{T}^{(0)}+D_{T}^{(1)} - F^{(0)}_T - F^{(1)}_F)^{\#} (h^{(0)}_T + h^{(1)}_{T})
\end{align*}
gives the estimation of $V$-function in \cref{eq:LSTD0-V}. Further, one can verify that with a plugging-in estimator for $Q$, the DQ estimator is given by 
\begin{align*}
\tauDK 
&= f(F^{(0)}_T, F^{(1)}_T, h^{(0)}_T, h^{(1)}_T)\\
&=: \one^{\top}(F^{(1)}_T-F^{(0)}_T)(D_{T}^{(0)}+D_{T}^{(1)} - F^{(0)}_T - F^{(1)}_F)^{\#} (h^{(0)}_T + h^{(1)}_{T})\\
&\quad + \one^{\top} (h^{(1)}_T - h^{(0)}_T).
\end{align*}
By Markov chain CLT, we have when $T$ goes to infinity
\begin{align*}
F^{(0)}_T \rightarrow F^{*}_0 := DP_0,\quad  F^{(1)}_T \rightarrow F^{*}_1 := DP_1\\
h^{(0)}_T \rightarrow h^{*}_0 := Dr_0,\quad  h^{(1)}_T \rightarrow h^{*}_1 := Dr_1
\end{align*}
where $D$ is a diagonal matrix with entries $D_{s,s} = \rho_{1/2}(s).$ Then by the delta method (see \cref{lem:mapping-CLT}), we have\footnote{The group inverse is continuous if we consider the set of matrices with rank $|\Sscr| - 1$ (\cite{rakovcevic1997continuity}).}
\begin{align*}
\sqrt{T} (f(F^{(0)}_T, F^{(1)}_T, h^{(0)}_T, h^{(1)}_T) - f(F^{*}_0, F^{*}_1, h^{*}_0, h^{*}_1)) \overset{d}{\rightarrow} N(0, \sigma_{f}^2)
\end{align*}
which is equivalent to 
\begin{align*}
\sqrt{T} (\tauDK - \E_{\rho_{1/2}}[\tauDK]) \overset{d}{\rightarrow} N(0, \sigma_{f}^2)
\end{align*}
since $ f(F^{*}_0, F^{*}_1, h^{*}_0, h^{*}_1) = \E_{\rho_{1/2}}[\tauDK].$ To analyze $\sigma_{f}$, we consider the ``linearization'' of $f$ around $u^{*} := (F^{*}_0, F^{*}_1, h^{*}_0, h^{*}_1).$ In particular, let $u(X_t) = (F^{(0)}(X_t), F^{(1)}(X_t), h^{(0)}(X_t), h^{(1)}(X_t)).$ Let $(\lambda, V)$ be the average reward and the ``true'' $V$-function under the policy $\pi_{1/2}.$ One can verify that
\begin{align*}
 \tilde{f}(s,a,s')&:=\nabla f(u^{*})^{\top}(u(s, a, s')-u^{*}) \\
 &= (\one^{\top}D(P_1-P_0)(I-P_{1/2})^{\#}D^{-1})E_{s} (r(s,a)-\lambda + V(s') - V(s)) \\
 &\quad + 2(1(a=1)-1(a=0))(V(s') + r(s, a)) - c
\end{align*}
where $c := \E_{\rho_{1/2}}[2(1(a=1)-1(a=0))(V(s') + r(s, a)-\lambda)].$ By \cref{lem:linearization}, we have
\begin{align*}
\sqrt{T}\left(\frac{1}{T}\sum_{t=1}^{T} \tilde{f}(X_t)\right) \overset{d}{\rightarrow} N(0, \sigma_{f}^2).
\end{align*}
Here $\sigma_{f}^2$ is explicitly given by (by Markov Chain CLT)
\begin{align*}
\sigma_{f}^2 
&:= \sum_{s,a,s'} \tilde{f}(s,a,s')^2 \rho_{1/2}(s)P_{a}(s, s') \frac{1}{2} \\
&\quad + 2\sum_{s,a,s'} \rho_{1/2}(s)P_{a}(s,s')\frac{1}{2}\tilde{f}(s,a,s') \sum_{s_1} (I-P_{1/2})^{\#}_{s',s_1} t(s_1)
\end{align*}
where $t(s) = \sum_{a,s'} \tilde{f}(s,a,s')P_{a}(s,s')\rho_{1/2}(s)\frac{1}{2}.$

%let $P^{(0)}_{T}, P^{(1)}_{T}$ be the empirical transition matrix of $\pi_0$ and $\pi_1$ derived from $F^{(0)}_T, F^{(1)}_T$, i.e., 
%\begin{align*}
%P^{(0)}_{T}(s, s') = \frac{F^{(0)}_T(s, s')}{\sum_{s''} F^{(0)}_T(s, s'')}, \quad P^{(1)}_{T}(s, s') = \frac{F^{(1)}_T(s, s')}{\sum_{s''} F^{(1)}_T(s, s'')}.
%\end{align*}

%note that
%\begin{align*}
%(F^{(0)} + F^{(1)})\hat{V}
%\end{align*}

\subsection{Bound $\sigma_{f}$}
Next, we aim to provide a bound for $\sigma_{f}.$ Note that that the mixing time of $X_{t}$ is the same as $s_{t}$ and by \cref{lem:bound-for-sigma}, we have
\begin{align*}
\sigma_{f} \leq \sqrt{2}\tilde{f}_{\max} \sqrt{\frac{2\ln(C)+1}{1-\lambda}}
\end{align*}
where $\tilde{f}_{\max} = \max_{s,a,s'} |\tilde{f}(s,a,s')|.$ Then the problem boils down to bound $\tilde{f}_{\max}.$ 

Let $z_{s} := (\one^{\top}D(P_1-P_0)(I-P_{1/2})^{\#}D^{-1})E_{s}.$ By the definition of $\tilde{f}$, we have
\begin{align*}
\tilde{f}_{\max} \leq 2(z_{\max}+2) (V_{\max} + r_{\max})
\end{align*}
where $z_{\max} := \max_{s} |z_{s}|, V_{\max} := \max_{s} |V(s)|.$ For $V_{\max}$, we have
\begin{align*}
\norm{V}_{\infty} 
&= \norm{(I-P_{1/2})^{\#} r}_{\infty}\\
&\leq \norm{(I-P_{1/2})}_{1,\infty} r_{\max}\\
&\leq \frac{2\ln(C)+1}{1-\lambda} r_{\max}.
\end{align*}
For $z_{\max}$, note that
\begin{align*}
z_{s} 
&= \one^{\top}D(P_1-P_0)(I-P_{1/2})^{\#}D^{-1})E_{s}\\
&= \rho_{1/2}^{\top} (P_1-P_0)(I-P_{1/2})^{\#} D^{-1} E_s\\
&= \rho_{1/2}^{\top} (P_1-P_0)(I-P_{1/2})^{\#}E_s \frac{1}{\rho_{1/2}(s)}.
\end{align*}
Then, we can invoke \cref{lem:non-expansive-lemma} to obtain that 
\begin{align*}
z_{s} = W(\rho_{1/2})(s) \frac{1}{\rho_{1/2}(s)} \leq 4\frac{\ln(C) +  \ln \left(1/\rho_{\min}\right) + 1}{1-\lambda}.
\end{align*}
Combining all together, we have
\begin{align*}
\sigma_{f} \leq C' \log\left(\frac{1}{\rho_{\min}}\right) \left(\frac{1}{1-\lambda}\right)^{5/2} r_{\max}
\end{align*}
for some constant $C'$ that depends (polynomially) on $\log(C)$, which completes the proof of \cref{th:dynkin_var}.

\subsection{Proof of \cref{lem:non-expansive-lemma}}
The only thing remaining is the proof of \cref{lem:non-expansive-lemma}. 

Note that we have $W(\rho^{1/2}) = (I-P_{1/2})^{\#\top}(P_1-P_0)^{\top} \rho_{1/2}.$ Let $v:= (P_1 - P_0)^{\top} \rho_{1/2}$. We will show first (i) $\frac{1}{2} v$ is entry-wise non-expansive; and then (ii) $(I-P_{1/2})^{\top}v$ is entry-wise bounded. 
 
To begin, we claim that $|(P_1-P_0)(s,s')| \leq 2 P_{1/2}(s,s')$ for any $s$ and $s'.$ This is due to $2P_{1/2} = P_{0} + P_{1}$ and for any $a\geq 0, b\geq 0$, we have $|a-b| \leq a+b$.

Furthermore, note that $\rho_{1/2}^{\top} P_{1/2} = \rho_{1/2}^{\top}.$ Then for any $s'$, 
\begin{align*}
|v(s')| 
&= \left|\sum_{s} \rho_{1/2}(s) (P_{1}-P_{0})_{s, s'}\right|\\
&\leq \sum_{s} \rho_{1/2}(s) | (P_{1}-P_{0})_{s, s'}|\\
&\leq \sum_{s} \rho_{1/2}(s) 2 P_{1/2}(s,s')\\
&\leq 2 \rho_{1/2}(s').
\end{align*}
This is to say, $\frac{v}{2}$ is entry-wise bounded by $\rho_{1/2}$. Furthermore, this bound continues to hold after any transformation for $v$ under $P_{1/2}$:
\begin{align*}
|(v^{\top}P_{1/2}^{k})(s')| 
&= \left|\sum_{s} v(s) P_{1/2}^{k}(s, s')\right|\\
&\leq 2\sum_{s} \rho_{1/2}(s) P_{1/2}^{k}(s, s')\\
&\leq 2\rho_{1/2}(s').
\end{align*}

Next, consider
\begin{align*}
v^{\top}(I-P_{1/2})^{\#}E_{s}
&= \sum_{k=0}^{\infty} v^{\top} (P_{1/2}^{k} - \one\rho_{1/2}^{\top})E_{s}\\
&=: \sum_{k=0}^{\infty} a_{k}.
\end{align*}
Note that $|(v^{\top}P_{1/2}^{k})E_{s}| \leq 2 \rho_{1/2}(s).$ Further, $|v^{\top}\one\rho_{1/2}^{\top}E_{s}| \leq |v^{\top} \one| \rho_{1/2}(s) \leq 2\rho_{1/2}(s).$ Therefore, for any $k$,  $|a_{k}| \leq 4\rho_{1/2}(s).$
We also have
\begin{align*}
|a_{k}| 
&\leq \norm{v^{\top}}_{1} \norm{P^{k}-\one\rho^{\top}}_{1,\infty} \norm{E_{s}}_{\max}\\
&\leq 2 C\lambda^{k}.
\end{align*}
Using the same trick in proving \cref{lem:group-inverse-L1}, we have
\begin{align*}
\frac{1}{\rho_{1/2}(s)}\sum_{k=0}^{\infty} |a_{k}| 
&\leq\sum_{k=0}^{\infty} \min\left(4, 2 C\lambda^{k} \frac{1}{\rho_{1/2}(s)}\right) \\
&\leq 2\left(\sum_{k=0}^{\log_{\lambda}(C/\rho_{1/2}(s))-1} 2 + \sum_{k=\log_{\lambda}(C/\rho_{1/2}(s))} \frac{C}{\rho_{1/2}(s)} \lambda^{k}\right)\\
&= \frac{4\ln(C/\rho_{1/2}(s))}{-\ln(\lambda)} + \frac{2}{1-\lambda}\\
&\leq \frac{4\ln(C/\rho_{1/2}(s))}{1-\lambda} + \frac{2}{1-\lambda}\\
&\leq 4\frac{\ln(C) +  \ln \left(1/\rho_{\min}\right) + 1}{1-\lambda}.
\end{align*}
Then $\left|W(\rho_{1/2})(s)\right| = |\sum_{k}a_{k}| \leq c \rho_{1/2}(s)$ with $c := 4\frac{\ln(C) +  \ln \left(1/\rho_{\min}\right) + 1}{1-\lambda}.$ This completes the proof.

\section{Proof of Theorem \ref{th:cramer-rao}}
The proof is based on multi-variate Cram\'{e}r-Rao bound. To begin, we assume $P_{0}(s, s') > 0,P_{1}(s, s') > 0$ for all $(s, s')$.\footnote{The general case follows a similar proof and is omitted for simplicity.}  

Consider the parameters $\theta = (F_0, F_1)$ which controls the transition matrices
\begin{align*}
P_0(s, s') = \frac{F_0(s, s')}{\sum_{s''} F_0(s, s'')}, \quad P_{1}(s, s') = \frac{F_1(s, s')}{\sum_{s''} F_1(s, s'')}. 
\end{align*}
Given the observations $X_{t} = (s_{t}, a_{t}), t=0, 1, \dotsc, T$ under the policy $\pi_{1/2}$. We can compute the log-likelihood 
\begin{align*}
l(X_{1},\dotsc, X_{T}~|~ \theta) = \left(\sum_{s,a,s'} n_{s,a,s'} \cdot \ln(P_{a}(s, s'))\right) - T\ln(2)
\end{align*}
where $n_{s,a,s'} = \sum_{t} 1(s_t=s, a_t=a, s_{t+1}=s').$ Then, the entry of the Fisher information matrix with $\theta^{*} = (P_0, P_1)$ is given by
\begin{align*}
I_{k,m} 
&= -\E_{X}\left[\frac{\partial  l(X|\theta^{*})}{\partial \theta_{k} \partial \theta_{m}}\right] \\
&= -\E_{X}\left[ \sum_{s,a,s'} \frac{n_{s,a,s'}}{P_{a}(s,s')} \cdot \frac{\partial P_{a}(s, s')}{\partial \theta_{k} \partial \theta_{m}} \right] + \E_{X}\left[ \sum_{s,a,s'} \frac{n_{s,a,s'}}{P_{a}(s,s')^2} \cdot \frac{\partial P_{a}(s, s')}{\partial \theta_{k}}  \frac{\partial P_{a}(s, s')}{\partial \theta_{m}}  \right] \\
&=  -T\sum_{s,a,s'} \frac{1}{2}\rho_{1/2}(s) \cdot \frac{\partial P_{a}(s, s')}{\partial \theta_{k} \partial \theta_{m}}  + T\sum_{s,a,s'} \frac{1}{2}\frac{\rho_{1/2}(s)}{P_{a}(s,s')} \cdot \frac{\partial P_{a}(s, s')}{\partial \theta_{k}}  \frac{\partial P_{a}(s, s')}{\partial \theta_{m}} \\
&= -T\frac{\partial 1}{\partial \theta_{k} \partial \theta_{m}} + T\sum_{s,a,s'} \frac{1}{2}\frac{\rho_{1/2}(s)}{P_{a}(s,s')} \cdot \frac{\partial P_{a}(s, s')}{\partial \theta_{k}}  \frac{\partial P_{a}(s, s')}{\partial \theta_{m}}\\
&=  T\sum_{s,a,s'} \frac{1}{2}\frac{\rho_{1/2}(s)}{P_{a}(s,s')} \cdot \frac{\partial P_{a}(s, s')}{\partial \theta_{k}}  \frac{\partial P_{a}(s, s')}{\partial \theta_{m}}.
\end{align*}
Consider $\theta_k = F_0(i, j), \theta_{m} = F_0(i, l)$, we have
\begin{align*}
\frac{1}{T} I_{k, m} = \frac{1}{2} \frac{\rho_{1/2}(i)}{P_0(i,j)} 1(j=l) - \frac{1}{2} \rho_{1/2}(i).
\end{align*}
For $\theta_{k} = F_1(i, j), \theta_{m} = F_1(i, l)$, we have
\begin{align*}
\frac{1}{T} I_{k, m} = \frac{1}{2} \frac{\rho_{1/2}(i)}{P_1(i,j)} 1(j=l) - \frac{1}{2} \rho_{1/2}(i).
\end{align*}
Otherwise it is easy to see that $I_{k,m} = 0.$ 

Next, consider an unbiased estimator $\hat{\tau}(X_1, \dotsc, X_{T})$ for $\rm{ATE}$. We can write ${\rm{ATE}} = f(F_0, F_1)$ as a function of $F_0$ and $F_1.$ Further, one can verify that
\begin{align*}
\frac{\partial f(\theta^{*})}{\partial F_0(i, j)} &= -\rho_0(i)(V_{\pi_0}(j) - V_{\pi_0}(i) + r_0(i) - \lambda^{\pi_0})\\
\frac{\partial f(\theta^{*})}{\partial F_1(i, j)} &= \rho_1(i)(V_{\pi_1}(j) - V_{\pi_1}(i) + r_1(i) - \lambda^{\pi_1}).
\end{align*} 
Finally, we aim to use the multi-variate Cram\'{e}r-rao bound. To do so, let $v^{(1)}_{i}$ be an vector with the $j$-th element being $v^{(1)}_i(j) = \rho_1(i)(V_{\pi_1}(j) - V_{\pi_1}(i) + r_1(i) - \lambda^{\pi_1})$. Let 
$$
I^{(1)}_{i}(j, l) = \frac{T}{2} \frac{\rho_{1/2}(i)}{P_1(i,j)} 1(j=l) - \frac{T}{2} \rho_{1/2}(i)
$$
be a matrix. Similarly, define $v^{(0)}_i$ and $I^{(0)}_{i}$ accordingly. Then, by the multi-variate Cram\'{e}r-rao bound for the singular Fisher information matrix \cite{stoica2001parameter}, we have
\begin{align*}
T{\rm Var}(\hat{\tau}) 
&\geq  \sum_{i} v^{(1) \top}_i (I^{(1)}_i)^{-1} v^{(1)}_i + \sum_{i} v^{(0) \top}_i (I^{(0)}_i)^{-1} v^{(0)}_i\\
&= 2\sum_{i} \frac{\rho_0(i)^2}{\rho_{1/2}(i)} \sum_{j} P_{0}(i, j) (V_{\pi_0}(j) - V_{\pi_0}(i) + r_0(i) - \lambda^{\pi_0})^2\\
&\quad + 2\sum_{i} \frac{\rho_1(i)^2}{\rho_{1/2}(i)} \sum_{j} P_{1}(i, j) (V_{\pi_1}(j) - V_{\pi_1}(i) + r_1(i) - \lambda^{\pi_1})^2
\end{align*}
which completes the proof. 

\subsection{Unbiased Estimator that achieves the lower-bound}
\label{sec:efficient-ope}

In this section, we construct an LSTD(0)-type OPE estimator that achieves the aforementioned Cram\'{e}r-Rao lower bound. To do so, we solve the following least square optimization problems that are similar to \cref{eq:LSTD0-V}, 
\begin{align}
(\hat{V}_1, \hat{\lambda}^{\pi_1}) &= \arg\min_{\hat V, \hat \lambda}
\sum_{s \in \Sscr} 
\left(\sum_{t, s_{t}=s, a_t=1} 
r(s_t, a_t) - \hat{\lambda} + \hat V(s_{t+1}) - \hat V(s_t)
\right)^2\\
(\hat{V}_0, \hat{\lambda}^{\pi_0}) &= \arg\min_{\hat V, \hat \lambda}
\sum_{s \in \Sscr} 
\left(\sum_{t, s_{t}=s, a_t=0} 
r(s_t, a_t) - \hat{\lambda} + \hat V(s_{t+1}) - \hat V(s_t)
\right)^2.
\end{align}
Then, the estimation for the average treatment effect is given by 
\begin{align*}
\tau_{{\rm{off}}} := \hat{\lambda}^{\pi_1} - \hat{\lambda}^{\pi_0}.
\end{align*}
To analyze the variance of $\hat{\tau}$, we follow the similar analysis as in Theorem \ref{th:dynkin_var}. To begin, one can verify that  
\begin{align*}
\hat{\lambda}^{\pi_0} - \lambda^{\pi_0} = \left(\hat{\rho}^{\top}_0 - \rho_0^{\top}\right)r_0
\end{align*}
where $\hat{\rho}_0$ is the empirical stationary distribution for the empirical transition matrix $\hat{P}_0$ ($\hat{\rho}_1$ and $\hat{P}_1$ can be defined accordingly).

Next, by the perturbation bound of $\hat{\rho}_0$, we have
\begin{align*}
    \hat{\rho}_0^{\top} - \rho_0^{\top} = \rho_0^{\top} (\hat{P}_0 - P_0) (I - \hat{P}_0)^{\#}.
\end{align*}
Hence,
\begin{align*}
\hat{\lambda}_0 - \lambda^{\pi_0} 
&= (\hat{\rho}_0^{\top} - \rho_0^{\top})r_0\\
&= \rho_0^{\top} (\hat{P}_0 - P_0) (I - \hat{P}_0)^{\#} r_0.
\end{align*}
Note that $\hat{P}_0$ is a function of $F^{(0)}_T$ ($\hat{P}_0(i, j) = F^{(0)}_T(i,j) / \sum_{k} F^{(0)}_T(i,k)$, $F^{(0)}$ is defined in \cref{eq:def-F-h}). Therefore, we can define $f_0(F^{(0)}_T) := \hat{\lambda}_0 - \lambda^{\pi_0}$ as a function of $F^{(0)}_T.$ Similarly, we can define 
$$
f_1(F^{(1)}_T) := \hat{\lambda}_1 - \lambda^{\pi_1} = \rho_1^{\top} (\hat{P}_1 - P_1) (I - \hat{P}_1)^{\#} r_1
$$
Then by \cref{lem:linearization}, we have the asymptotic normality for $\eoff$:
$$
\sqrt{T}(\eoff - {\rm{ATE}}) = \sqrt{T}(f_1(F^{(1)}_T) - f_0(F^{(0)}_T)) \overset{d}{\rightarrow} N(0, \sigoff^2).
$$
In order to compute $\sigoff$ by using \cref{lem:linearization}, we will linearize $f_1-f_0$ around $(F_0^{*}, F_1^{*})$. To do so, consider
\begin{align*}
    \frac{\partial f_0(F_0)}{\partial (F_0)(i, j)} &= \rho_0^{\top} \frac{\partial (\hat{P}_0 - P_0)}{\partial F_0(i, j)} (I - P_0)^{-1} (r_0 - \lambda^{\pi} \one) \\
    &\quad +  \rho_0^{\top} (P_0 - P_0) \frac{\partial (I - P_0)^{-1}}{\partial (F_0)(i,j)} (r_0 - \lambda^{\pi} \one)\\
    &= \rho_0^{\top} \frac{\partial \hat{P}_0}{\partial F_0(i, j)} V_0\\
    &= \sum_{k} \rho_0(i) V_0(k)  \frac{\partial \hat{P}_0(i,k)}{\partial F_0(i, j)} 
\end{align*}
Note that $\hat{P}(i, k) = \hat{F}_0(i,k) / \sum_{l} \hat{F}_0(i, l)$. Therefore,
\begin{align*}
    \frac{\partial f_0(F_0)}{\partial (F_0)(i, j)}
    &= \sum_{k} \rho_0(i) V_0(k) \frac{\partial \frac{F_0(i,k)}{\sum_{l} F_0(i, l)}}{\partial F_0(i, j)} \\
    &= \sum_{k} \rho_0(i) V_0(k) \frac{1(j=k)\sum_{l} F_0(i, l) - F_0(i, k)}{(\sum_{l} F_0(i, l))^2} \\
    &= \sum_{k} \rho_0(i) V_0(k) \frac{1(j=k)\rho(i) - \rho(i)P_0(k|i)}{\rho(i)^2}\\
    &= \frac{\rho_0(i)}{\rho(i)}V_0(j) - \frac{\rho_0(i)}{\rho(i)}\sum_{k} P_0(k|i)V_0(k)\\
    &= \frac{\rho_0(i)}{\rho(i)}(V_0(j) - V_0(i) + r_0(i) - \lambda^{\pi_0}). 
\end{align*}
Hence, the linearization of $f_0$ is 
\begin{align*}
&\sum_{ij} \frac{\partial f_0(F_0)}{\partial (F_0)(i, j)} \left(\left(F_0(s, s', a)\right)_{ij} - F_0(i,j) \right)\\
&= 2\cdot 1(a=0) \frac{\rho_0(s)}{\rho(s)}(V_0(s') - V_0(s) + r_0(s) - \lambda^{\pi_0}) - \sum_{ij}\rho_0(i)(V_0(j)P_0(j|i) - V_0(i) + r_0(i) - \lambda^{\pi_0})\\
&=  2\cdot 1(a=0) \frac{\rho_0(s)}{\rho(s)}(V_0(s') - V_0(s) + r_0(s) - \lambda^{\pi_0}).
\end{align*}
The similar linearization can be done for $f_1.$ Then the linearization of $f_1-f_0$ is
\begin{align*}
    g((s, s', a)) 
    &= -2\cdot 1(a=0) \frac{\rho_0(s)}{\rho(s)}(V_0(s') - V_0(s) + r_0(s) - \lambda^{\pi_0})\\ 
    &\quad + 2\cdot 1(a=1) \frac{\rho_1(s)}{\rho(s)}(V_1(s') - V_1(s) + r_1(s) - \lambda^{\pi_1}).
\end{align*}
Note that for any $E[g(X_k)|X_1=(s,s',a)] = 0$ for any $(s, s', a)$ and $k\geq 2$. Hence
\begin{align*}
    \sigoff^2 
    &={\rm{Var}}_{\rho}(g) + 2\sum_{k=2}^{\infty} {\rm{Cov}}_{\rho}[g(X_{k})g(X_1)] \\
    &={\rm{Var}}_{\rho}(g) \\
    &=2 \sum_{s,s'} \frac{\rho_0(s)^2 P_0(s'|s)}{\rho(s)}(V_0(s') - V_0(s) + r_0(s) - \lambda^{\pi_0})^2\\
    &\quad + 2 \sum_{s,s'} \frac{\rho_1(s)^2 P_1(s'|s)}{\rho(s)}(V_1(s') - V_1(s) + r_1(s) - \lambda^{\pi_1})^2
\end{align*}
which completes the proof. 

\section{Proof of Theorem \ref{th:price}}
We construct a birth-death Markov chain with $n$ states. Let $P \in \R^{n\times n}$ be a transition matrix where $P(s, s+1) = \frac{1}{4} - \delta, P(s, s-1) = \frac{1}{4}$ and $P(s, s) = 1/2+\delta$ (exception at two ends with $P(0, 0)=3/4+\delta$ and $P(n-1,n-1) = 3/4$).

Let the stationary distribution of $P$ be $\rho.$ Then $\rho(s) = c\left(1-4\delta\right)^{s}$ for $0\leq s \leq n-1$ and $c := \frac{1}{\sum_{s} \left(1-4\delta\right)^{s}}$ is a constant. By \cite{chen2013mixing}, we have the spectral gap of the chain is in the order of $\gamma = O(1/n)$. Furthermore, the mixing time of the chain is bounded by
\begin{align*}
\norm{P^{k} - \one \rho^{\top}}_{1,\infty} 
&\leq \left(\frac{1}{\rho_{\min}}\right) (1-\gamma)^{k}\\
\norm{(I-P)^{\#}}_{1,\infty} 
&\leq \log \left(\frac{1}{\rho_{\min}}\right)O(n) = O(n^2).
\end{align*}
Following the same proof in \cref{th:dynkin_var}, we have the on-policy variance is bounded by 
\begin{align*}
\sigon = O(n^{6}).
\end{align*}
On the other hand, consider the node $k$ where $\sum_{s=k}^{n} \rho(s) \leq c'\delta / n^2$ and $\sum_{s=k-1}^{n} \rho(s) > c'\delta / n^2$ for some sufficient small constant $c'$. Let $P_1$ be the same as $P$ except $\forall s\geq k$
\begin{align*}
P_1(s, s+1) &= \frac{1}{4} \\
P_1(s, s) &= \frac{1}{2}.
\end{align*}
Let $\rho_1$ be the stationary distribution of $P_1.$ One can verify that $\rho_1(n) = O(1/n^2).$ We then construct $r$ such that $r(n, 1) = 1$ and $\lambda^{\pi_1} = 0.$ Then
\begin{align*}
\sigoff 
&\geq \sqrt{2\frac{\rho_1(n)^2}{\rho(n)}\frac{3}{4}} \\
&= \Omega\left(\frac{e^{cn}}{n^2}\right)
\end{align*}
for some constant $c$. Therefore,
\begin{align*}
\frac{\sigon}{\sigoff} = O\left(\frac{n}{e^{c' n}}\right)
\end{align*}
for some constant $c'.$ Next, consider the bias of DQ estimator. Suppose ${\rm ATE}=\delta$ without loss (one can always achieve this by adding some constants to $r$). Let $P_0 = 2 \cdot P_1 - P$ and $\rho_0$ be the stationary distribution of $P_0.$ One can verify that
\begin{align*}
\norm{\rho_1 - \rho}_{1} = O(\delta/n^2), \norm{\rho_0-\rho}_1 = O(\delta/n^2).
\end{align*}
Furthermore, following the proof in \cref{th:dynkin_bias}, we have
\begin{align*}
|(\mathrm{ATE}-\E[\hat{\mathrm{ATE}}_{\rm DQ}])/\mathrm{ATE}|
&\leq (\norm{\rho_1 - \rho}_{1} + \norm{\rho_0 - \rho}_{1}) \norm{I-P}^{\#}_{1,\infty} \\
&\leq C\cdot c'\delta \frac{1}{n^2} n^2\\
&\leq \delta 
\end{align*}
for sufficient small constant $c'.$ This completes the proof. 

\section{Technical Lemmas}
%\begin{lemma}[Stationary Distribution Perturbation, Theorem 4.1 \cite{meyer1980condition}]\label{lem:perturbation-stationary-distribution}
%Suppose $P \in \R^{n \times n}$ and $P' \in \R^{n\times n}$ are transitions matrices of two finite-state aperiodic and irreducible Markov Chains and $\rho \in \R^{n}, \rho' \in \R^{n}$ are the stationary distributions accordingly. Then
%\begin{align*}
%\rho'^{\top} - \rho^{\top} = \rho'^{\top} (P' - P) (I - P)^{\#}.
%\end{align*} 
%%\end{lemma}
%
\begin{replemma}{lem:group-inverse-L1}
    Suppose $P \in \R^{n\times n}$ is the transition matrix of a finite-state aperiodic and irreducible Markov Chain and $\rho$ is the stationary distribution. Suppose there exists $C$ and $\lambda$ such that for any $k=0,1,\dotsc$
    $$
    \norm{P^{k}-\one\rho^{\top}}_{1,\infty} \leq C \lambda^{k}.
    $$
    Then
    \begin{align*}
        \norm{(I-P)^{\#}}_{1,\infty} \leq \frac{2\ln(C)+1}{1-\lambda}.
    \end{align*}
\end{replemma}
\begin{proof}
Note that 
\begin{align*}
    A &= (I-P+1\rho^{\top})^{-1} - 1\rho^{\top}\\
    &= \sum_{k=0}^{\infty} \left( P^{k} - 1\rho^{\top}\right).
\end{align*}
Then
\begin{align*}
    \norm{A}_{1,\infty} 
    &\leq \sum_{k=0}^{\infty} \norm{P^{k} - 1\rho^{\top}}_{1,\infty}\\
    &\leq \sum_{k=0}^{\infty} \min\left(2, C \lambda^{k}\right)\\
    &\leq \sum_{k=0}^{\log_{\lambda}(1/C)-1} 2 + \sum_{k=\log_{\lambda}(1/C)}^{\infty} C\lambda^{k}\\
    &\leq 2\log_{\lambda}(1/C) + \frac{1}{1-\lambda}\\
    &= 2\frac{\ln(C)}{-\ln(\lambda)} + \frac{1}{1-\lambda}\\
    &\overset{(i)}{\leq} \frac{2\ln(C)+1}{1-\lambda}
\end{align*}
where (i) is due to $-\ln(x) \leq 1 - x$ for $x>0.$
\end{proof}

\begin{lemma}\label{lem:bound-for-sigma}
    For a finite-state aperiodic and irreducible Markov Chain $X_1, X_2, \dotsc, X_{t}$. Let $P$ be the transition matrix, $\rho$ be the stationary distribution, and $\Sscr$ be the state space. Suppose there exists $C$ and $\lambda$ such that for $k=0,1,\dotsc$,
    $$
    	\norm{P^{k} - \one\rho^{\top}}_{1,\infty} \leq C \lambda^{k}.
    $$ 
Then for any bounded function $f: \Sscr \rightarrow [a, b]$, there exists $\sigma$ such that when $T$ goes to infinity, 
    \begin{align} \label{eq:CLT}
        \frac{1}{\sqrt{T}} \sum_{t=1}^{T} \left(f(X_t) - f^{*}\right) \overset{d}{\rightarrow} N(0, \sigma^2)
    \end{align}
    where $f^{*} = \E_{\rho}(f)$ is the expected value of $f$ under the stationary distribution and 
    \begin{align}\label{eq:variance-bound}
        \sigma \leq \sqrt{2}(b-a)\sqrt{\frac{2\ln(C)+1}{1-\lambda}}.
    \end{align}
\end{lemma}
\begin{proof}
Note that \cref{eq:CLT} is simply due to the Markov chain CLT (\cite{jones2004markov}). Let $D$ be an diagonal matrix with entries $D_{ii} = \rho_{i}.$ \cite{jones2004markov} further states that
\begin{align*}
\sigma^2 
&= {\rm Var}_{\rho}(f) + 2\sum_{k=2}^{\infty} \E_{\rho}[(f(X_{1})-f^{*})(f(X_{k}) - f^{*})]\\
&= (f - f^{*})^{\top} D (f-f^{*}) + 2\sum_{k=1}^{\infty} (f-f^{*})^{\top} D P^{k} (f-f^{*})\\
&=  2\sum_{k=0}^{\infty} (f-f^{*})^{\top} D (P^{k} - \one \rho^{\top}) (f-f^{*}) - (f - f^{*})^{\top} D (f-f^{*})\\
&\leq 2\sum_{k=0}^{\infty} (f-f^{*})^{\top} D (P^{k} - \one \rho^{\top}) (f-f^{*})\\
&\leq 2 \norm{(f-f^{*})^{\top} D}_{1} \norm{I-P}^{\#}_{1,\infty} \norm{f-f^{*}}_{\max}\\
&\overset{(i)}{\leq } 2\norm{f-f^{*}}_{\max}^2 \frac{2\ln(C)+1}{1-\lambda}.
\end{align*}
where (i) is due to \cref{lem:group-inverse-L1}. Therefore, 
\begin{align*}
\sigma \leq \sqrt{2}(b-a) \sqrt{\frac{2\ln(C)+1}{1-\lambda}}.
\end{align*}
\end{proof}

\begin{lemma}[Theorem 6.2 \cite{konda2002actor}]\label{lem:mapping-CLT}
    Let $U_{k}$ be a sequence of random variables in $\R^{p}$ converging in probability to $u$. Let $a_{k}$ be a deterministic non-negative sequence increasing to $\infty.$ Let $\sqrt{\alpha_{k}}(U_k - u)$ converge in distribution to $N(0, \Gamma).$ Let $f: R^{p} \rightarrow R^{q}$ be a function twice differentiable in a neighborhood of $u$. Then, denoting the Jacobian of $f$ at $u$ by $\nabla f(u)$, we have
    \begin{enumerate}
        \item $f(U_k)$ converges in probability to $f(u).$
        \item $\sqrt{\alpha_k}(f(U_{k}) - f(u))$ converges in distribution to $N(0, \nabla f(u^{*}) \Gamma \nabla f(u^{*})^{\top}).$
    \end{enumerate}
\end{lemma}

\begin{lemma}\label{lem:linearization}
Consider an irreducible and aperiodic finite-state space Markov Chain $X_1, X_2, \dotsc, X_{t}$. Let $S$ be the state space and $\rho$ be the stationary distribution. Let $u: S \rightarrow \R^{p}$ be a function with each component $u_{i}, 1\leq i \leq p.$ Let $u^{*} = \sum_{s \in S} \rho(s) u(s)$ be the expected value of $u$ under the stationary distribution $\rho.$

Let $f: \R^{p} \rightarrow \R$ be a function twice differentiable in a neighbor of $u^{*}.$ Then, there exists $\sigma \geq 0$ such that when $T \rightarrow \infty$, 
\begin{align*}
    \sqrt{T}\left(f\left(\frac{1}{T}\sum_{i=1}^{T} u(X_t)\right) - f(u^{*})\right) &\overset{d}{\rightarrow} N(0, \sigma^2)\\
    \sqrt{T} \left(\sum_{i=1}^{p} \left(u_i(X_t)-u^{*}_i\right) \cdot \frac{\partial f(u^{*})}{\partial u_i}  \right) &\overset{d}{\rightarrow} N(0, \sigma^2)
\end{align*}
\end{lemma}
\begin{proof}
 To begin, note that by Markov Chain CLT (Corollary 5 \cite{jones2004markov}), we have
 \begin{align*}
     \sqrt{T} \left(\frac{1}{T}\sum_{i=1}^{T} u(X_t) - u^{*}\right) \overset{d}{\rightarrow} N(0, \Sigma)  
 \end{align*}
 for some covariance matrix $\Sigma \in \R^{p \times p}.$ In particular,
 \begin{align}\label{eq:def-Sigma}
     \Sigma := E_{\rho}[(u(X_1)-u^{*})(u(X_1)-u^{*})^{\top}] + 2\sum_{k=2}^{\infty}  E_{\rho}[(u(X_1)-u^{*})(u(X_{k})-u^{*})^{\top}]
 \end{align}
 where $E_{\rho}$ denotes the expectation when the initial distribution of the Markov chain is $\rho$. 
 
 Then, since $f$ is twice differentiable in a neighbor of $u^{*}$, we can invoke \cref{lem:mapping-CLT} to get 
 \begin{align*}
      \sqrt{T}\left(f\left(\frac{1}{T}\sum_{i=1}^{T} u(X_t)\right) - f(u^{*})\right) &\overset{d}{\rightarrow} N(0, \sigma^2)
 \end{align*}
 where $\sigma^2 := \nabla f(u^{*})^{\top} \Sigma \nabla f(u^{*}).$

 Next, let $F(X) := \sum_{i=1}^{p} \left(u_i(X)-u^{*}_i\right) \cdot \frac{\partial f(u^{*})}{\partial u_i} = (u(X) - u^{*})^{\top} \nabla f(u^{*})$. Then using 
the fact $\frac{1}{T}\sum_{t=1}^{T}u(X_{t}) \rightarrow u^{*}$ and invoking Markov Chain CLT again, we have
 \begin{align*}
     \sqrt{T}\left(\frac{1}{T}\sum_{t=1}^{T}F(X_{t})\right) &\overset{d}{\rightarrow} N(0, \sigma_{F}^2)
 \end{align*}
 where
 \begin{align*}
     \sigma_{F}^2 := E_{\rho}[F(X_1)^2] + 2\sum_{k=2}^{\infty} E_{\rho}[F(X_1)F(X_{k})]. 
 \end{align*}
 Expanding $F(X)$ by $(u(X) - u^{*})^{\top} \nabla f(u^{*})$, we have
 \begin{align*}
     \sigma_{F}^2 
     &= E_{\rho}[((u(X_1) - u^{*})^{\top} \nabla f(u^{*}))^2] + 2\sum_{k=2}^{\infty} E_{\rho}[(u(X_1) - u^{*})^{\top} \nabla f(u^{*}) (u(X_k) - u^{*})^{\top} \nabla f(u^{*})]\\
     &= \nabla f(u^{*})^{\top} E_{\rho}[(u(X_1)-u^{*})(u(X_1)-u^{*})^{\top}] \nabla f(u^{*}) \\
     &\quad + \nabla f(u^{*})^{\top} \sum_{k=2}^{\infty} E_{\rho}[(u(X_1)-u^{*})(u(X_k)-u^{*})^{\top}] \nabla f(u^{*}) \\
     &\overset{(i)}{=}  \nabla f(u^{*})^{\top} \Sigma \nabla f(u^{*})\\
     &= \sigma^2 
 \end{align*}
 where (i) uses \cref{eq:def-Sigma}. This implies that $F$ (the linearization of $f$ at the point $u^{*}$) will converge with the same limiting variance as $f$.  
\end{proof}

%%% Local Variables:
%%% mode: latex
%%% TeX-master: "main.tex"
%%% End:

\section{Experiment details}

\subsection{Synthetic example}
\subsubsection{Environment}
We replicate exactly the environment of \cite{johari2022experimental}. We model a rental marketplace with $N=5000$ homogeneous listings.  Customers arrive according to a Poisson process with rate $N \lambda$, decide whether to rent a listing (with rental probability controlled by the intervention), and if they do rent, they occupy a listing for an exponentially distributed time with mean $\frac{1}{\mu}$.

Specifically, we define our MDP to be the discrete-time jump chain of this process, with events indexed by $t$ and state $s_t \in \{0, 1 \ldots N \}$ representing the current inventory of listings. At the $t^{\rm th}$ event, the system chooses to apply control ($a_t = 0$) or treatment ($a_t = 1$). One of the following state transition and reward scenarios may then happen:

\begin{enumerate}
  \item A previously occupied rental becomes available, i.e. $s_{t+1} = s_t + 1$ and $r_t = 0$; this occurs with probability $\frac{(N - s_t) \mu}{N\mu + N \lambda}$
  \item A customer arrives, with probability $\frac{N\lambda}{N\mu + N\lambda}$, and subsequently:
        \begin{enumerate}
          \item Rents a listing, so $s_{t+1} = s_t - 1$ and $r_t = 1$; this occurs with probability $\frac{s_t v(a_t)}{N + s_t v(a_t)}$ where $v(0) = 0.315$ and $v(1) = 0.3937$ are the average utility under control and treatment, respectively.
          \item Does not rent a listing , so $s_{t+1} = s_t$ and $r_t = 0$; this occurs with probability $\frac{N}{N + s_t v(a_t)}$.
        \end{enumerate}
  \item No state change occurs; i.e. $s_{t+1} = s_t$ and $r_t = 0$.
\end{enumerate}

\cite{johari2022experimental} also describes a two-sided randomization scheme, where listings are also assigned to control or treatment, and the customer's purchase probability depends on both the customer's treatment assignment $a_t$, as well as the number of control listings and the number of treatment listings. This corresponds to a more complicated MDP with a two-dimensional state $s_t = (s_t^{\rm co}, s_t^{\rm tr})$, where $s_t^{\rm co}$ corresponds to the number of available control listings, and $s_t^{\rm tr}$  the number of available treatment listings. The average utility of a control listing is $v_{\rm co}(0) = v_{\rm co}(1) = v(0)$, while the average utility of a treatment listing is $v_{\rm tr} (0) = v(0)$ and $v_{\rm tr}(1) = v(1)$. We defer to \cite{johari2022experimental} for further details of this scheme.

\subsubsection{Implementation details}
Here we list algorithms and hyperparameters tuned for this experiment. Hyperparameters were chosen to minimize MSE averaged over 10 held-out trajectories. As in \cite{johari2022experimental}, we also include a burn-in period of $T_0=5N$.

\begin{enumerate}
  \item Naive. This has no hyperparameters.
  \item TSRI. This has several hyperparameters, which affect both the experimental design (customer randomization probability $p$ and listing randomization probability $p_{L}$), as well as the estimator (parameters $k$ and $\beta$, as described in \cite{johari2022experimental}). We set $p, p_{L}, \beta$ assuming $\lambda, \mu$ are known, exactly as prescribed in \cite{johari2022experimental}. Specifically, we compute the values reported in Table~\ref{tbl:synthetic-hyperparams} as:
        \begin{align*}
p=\left(1-e^{-\lambda / \mu}\right)+ 0.5  e^{-\lambda / \mu} && p_{L}=0.5\left(1-e^{-\lambda / \mu }\right)+e^{-\lambda / \mu }  &&  \beta = e^{-\lambda / \mu}
        \end{align*}
        We report results for both $k=1$ and $k=2$.
  \item DQ with LSTD, which we estimate using a slight modification of Equation~\eqref{eq:LSTD0-V}. Specifically, we directly estimate the state-action value function $Q$ instead of separately estimating the state value function $V$ and $P_{1}, P_{0}$, and we add an $L_{2}$ regularization term. In short, we approximate and solve for a fixed point to the regularized least-squares problem:
        \begin{equation*}
Q = \arg \min_{Q'} \|Q' - r - PQ + \lambda\|^2_{2} + \alpha \|Q'\|^2_{2}
        \end{equation*}
        where $Q \in \mathbb{R}^{2(N+1)}$ is the vector of estimated $Q(s,a)$ values, and$P \in \mathbb{R}^{2(N+1) \times 2(N+1)}$ is the state-action transition matrix. We use sample means in each state to construct plug-in estimates of $r, P$ and $\lambda$.

  \item Off-Policy with LSTD, which we note is novel in the literature. In Section~\ref{sec:efficient-ope} we describe this algorithm, provide convergence guarantees, and show that this algorithm is efficient. This can be construed as a direct analog of \cite{shi2020reinforcement}'s off-policy estimator, which applies LSTD in the discounted-reward setting. It has no hyperparameters.
  % \item Off-Policy (LSTD), where $Q$ -functions and off-policy average rewards $\lambda_{0}$ and $\lambda_{1}$ estimated via LSTD. We note that this estimator is novel in the literature, but provides a meaningful comparison because 1) we can show that it is convergent and efficient (see Section~\ref{sec:efficient-ope}); 2) it constitutes an average-reward analog of \cite{shi2020reinforcement}'s off-policy estimator, which applies LSTD in the discounted-reward setting. This has no hyperparameters.
  \item Off-Policy with TD, where $Q$ -functions and off-policy average rewards are calculated according to the Differential TD algorithm of \cite{wanLearningPlanningAverageReward2021}. This approach has two hyperparameters: the learning rate for the $Q$ -function $\gamma / \sqrt{t}$, and the learning rate for the mean reward estimate $\beta \gamma / \sqrt{t}$.
\end{enumerate}

For these experiments, we exclude the Off-Policy GTD variant described in \cite{zhangAverageRewardOffPolicyPolicy2021} as their convergence guarantees do not apply to the tabular setting.

\begin{table}[htbp]
  \centering
  \begin{tabular}{r| l }
    Algorithm & Hyperparameters \\
    \hline
    TSRI & $p = {\bf 0.816}, p_{L} = {\bf 0.683}, k \in \{ {\bf 1, 2} \}, \beta = {\bf 0.368}$  \\
    DQ (LSTD) & $\alpha \in \{ 0.01, {\bf 0.1}, 1, 10, 100\}$  \\
Off-Policy (TD) & $\beta \in \{0.2, {\bf 0.5} \}, \gamma \in \{ 0.001, {\bf 0.01}, 0.1, 1.\}$
  \end{tabular}
  \caption{Hyperparameters for the synthetic example of \cite{johari2022experimental}. Parameter settings reported in the main text are in bold.}
  \label{tbl:synthetic-hyperparams}
\end{table}

\subsubsection{Additional results} We note that there are scenarios for which which specialized designs and  estimators -- specifically TSR, in this example -- can provide a superior bias-variance tradeoff. \cite{johari2022experimental} shows that the TSRI estimators become unbiased when $\lambda \gg \mu$. We ran the synthetic example setting $\lambda=10, \mu=1$ (also mirroring results from \cite{johari2022experimental}), and indeed for this setting for reasonable horizons TSR achieves lower RMSE. Recall, however, that TSR is ill-defined for settings where there is no natural notion of two-sided randomization (i.e. in any MDP without a notion of two sides), and its bias properties are clearly highly instance-specific and depend on knowledge of $\lambda, \mu$. DQ still outperforms all alternatives besides TSR in this setting,  and even in this extremely unbalanced setting bachieves a much lower asymptotic bias than TSR (-5e-3 vs 1e-2, as a proportion of the treatment effect magnitude). .

\begin{figure}[htbp]
 \centering
 \subfloat{\includegraphics[width=0.46\textwidth]{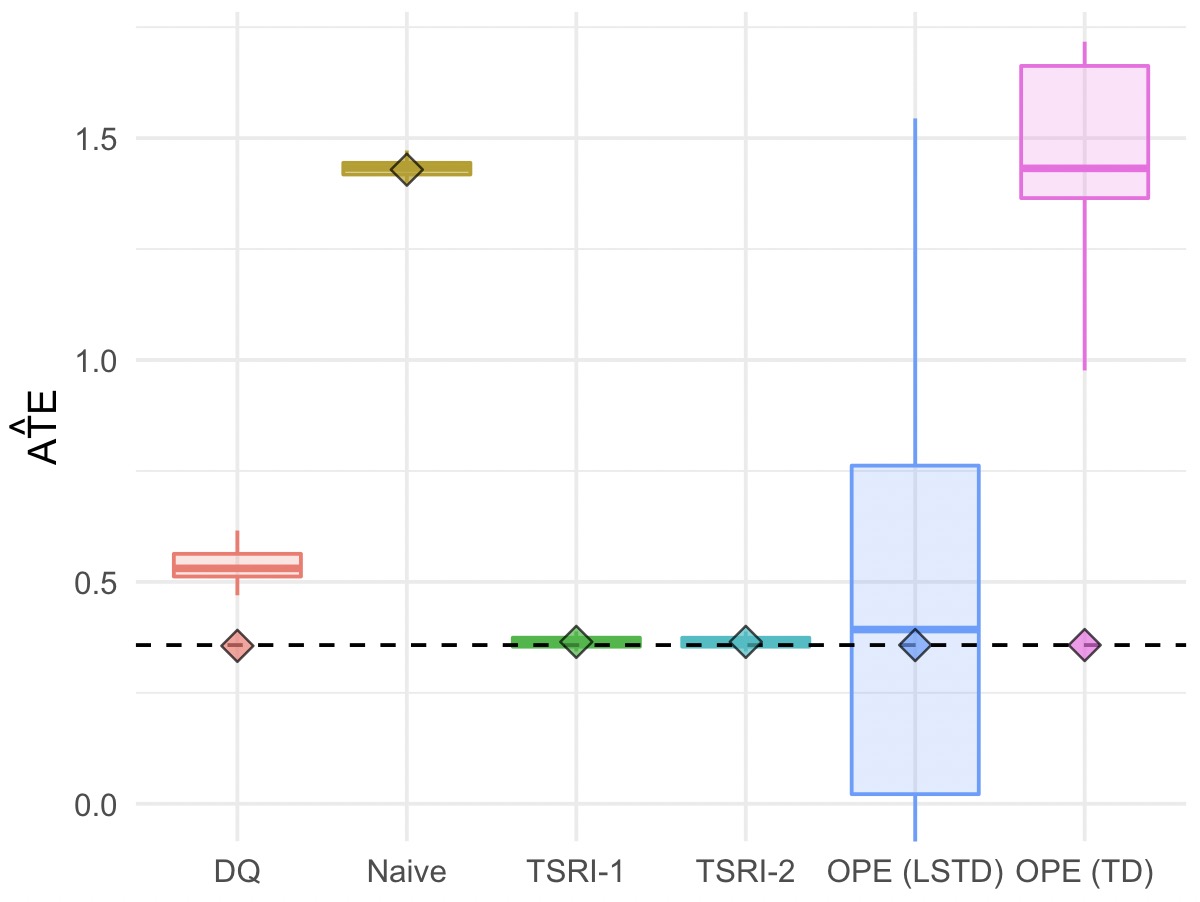}}
 \subfloat{\includegraphics[width=0.5\textwidth]{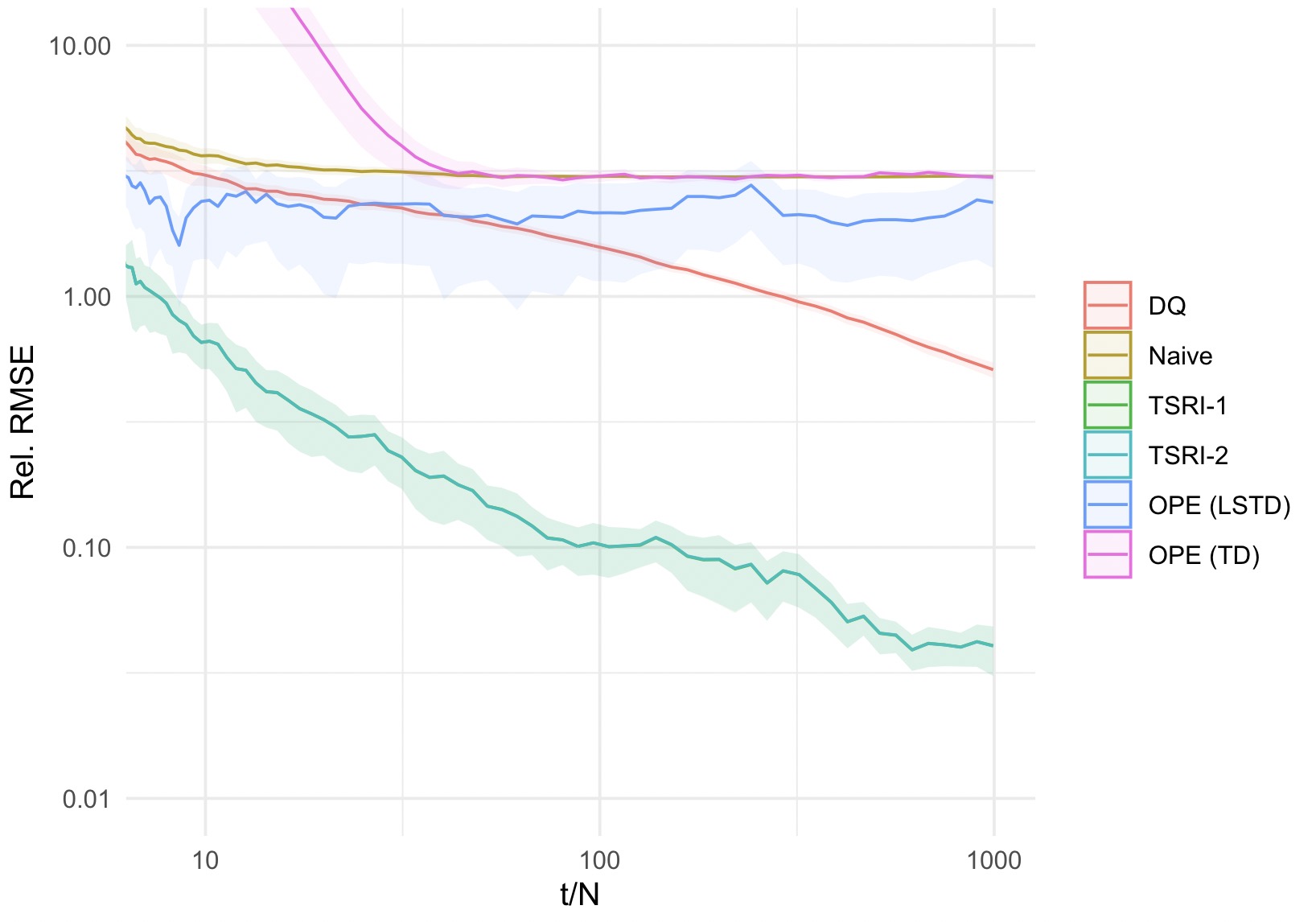}}
 \caption{Simple example from \cite{johari2022experimental}, with $\lambda=10$. {\it Left}: Estimated ATE at time $t/N=10^{3}$ across $100$ trajectories. Dashed line indicates actual ATE. Diamonds indicate the asymptotic mean for each estimator. Over this horizon, TSRI-1 and TSRI-2 exhibit small bias and variance, although asymptotically DQ still has lower bias.}
 \label{fig:high-lambda}
\end{figure}

\subsubsection{Computing environment} These experiments were performed on a personal desktop with a 24-core Intel Xeon X5670 CPU and 128 GB RAM. Total compute time per seed averaged less than two hours.

\subsection{Ridesharing Simulator}
\subsubsection{Environment}

We implement a ridesharing simulator, with code available on Github.

\begin{enumerate}
  \item Riders are generated based on trips resampled from the NYC Taxi Dataset \cite{TLCTripRecord} (specifically, from January 11, 2015), with a random willingness-to-pay per second distributed as $\mathrm{LogNormal}(\log(0.01), 1.)$. The the rider's outside option is assumed to be the trip they actually took in the dataset, and the cost (i.e., negative utility) the rider incurs for this option is the fare recorded in the dataset, plus the trip time times the rider's WTP per second.
  \item Drivers enter the system at pickup locations in the same dataset, but at a lower arrival rate (tuned to achieve a utilization of $\sim70\%$). Drivers stay in the system for an exponential time with a mean of one hour, and stop serving new requests once they exit the system.
  \item When a request enters the system, the pricing engine computes the cost to serve that request with an idle driver (where cost is based on recent per-mile and per-minute fare rates), and discounts this by 10\%; this is the price offered to the rider. The pricing engine also offers the rider a worst-case time-to-destination (ETD) guarantee, which is 1.5 times the time to serve the request with an idle driver. The rider then chooses to accept or reject the offer, based on whether their worst-case utility for the trip exceeds the utility of the outside option. If the rider rejects the offer they exit the system.
  \item If the rider accepts, the request is submitted to the dispatch engine. The dispatcher searches for the nearest idle driver and the 10 nearest pool drivers to the request. This list of candidates is filtered to those who can serve the request while satisfying the ETD guarantees of all riders. The pool candidates are then further filtered to those whose cost to service the request is at most $\frac{1}{1 + \alpha_t}$ times the cost of the idle driver, where $\alpha_t = \alpha_{\rm co} = 0$ in control ($a_t = 0$) and $\alpha_t = \alpha_{\rm tr}$ in treatment ($a_t = 1$), where we vary $\alpha_{\rm tr} \in \{0.3, 0.5, 0.7\}$. Finally, the minimum cost driver among this set is dispatched.
\end{enumerate}

We can implement two-sided randomization in this market as follows. Each driver is also randomized into either treatment or control. The dispatcher then dispatches to the minimum cost driver among the following set:

\begin{itemize}
  \item All idle drivers (i.e., drivers currently assigned no passengers).
\item Control pool drivers, whose cost is at most $\frac{1}{1 + \alpha_{\rm co}}$ times the minimum cost idle driver.
\item Treatment pool drivers, whose cost is at most $\frac{1}{1 + a_t \alpha_{\rm tr} + (1 - a_t) \alpha_{\rm co}}$ times the minimum cost idle driver.
\end{itemize}

\subsubsection{Algorithms}
We use the same approximation architecture for each algorithm, where $Q(s, a) = \theta^\top\phi(s, a)$ is a linear function of features $\phi: \mathcal{S} \times \mathcal{A} \mapsto \mathbb{R}^{d}$ with coefficients $\theta$. We take features $\phi(s_t, a_t)$ to consist of the number of drivers in the system with each of 0, 1, 2, and 3 open seats remaining, as well as the price and cost of the current request.

The algorithms are then:

\begin{enumerate}
  \item Naive, with no hyperparameters.
  \item TSRI, again with hyperparameters $p, p_{L}, k, \beta$. We set these based on the relative supply and demand characteristics of the simulator. Specifically, with analogy to the synthetic problem, the system averages around $600$ drivers, with 3 passenger seats per driver, for a total of $N \approx 1800$ available units of capacity. The arrival rate is 4 passengers per second, yielding $\lambda \approx 4 / 1800$, while the average trip lasts 12 minutes, yielding $\mu \approx 720$. Ultimately we have $\lambda / \mu \approx 1.6$, and set the algorithm hyperparameters accordingly.
  \item DQ with LSTD, with a single regularization hyperparameter $\alpha$. Here we solve for $\theta$ by approximating and solving for a fixed point to the regularized least-squares problem:
        \begin{equation*}
\theta = \arg \min_{\theta'} \|\Phi\theta' - r - P\Phi\theta + \lambda\|^2_{2} + \alpha \|\theta'\|^2_{2}
        \end{equation*}
        where $\Phi \in \mathbb{R}^{|\mathcal{S}| \times |\mathcal{A}|}$ is the matrix of state-action feature representations.
  \item Off-Policy with LSTD, where we solve simultaneously for $\theta_{1}, \lambda_{1}$ by solving for the unique fixed point of the projected Bellman equation $\Phi_{1}^\top\Phi_{1} \theta_{1} = \Phi_{1}^\top (r_{1} - {\bf 1}\lambda_{1}) + \Phi_{1}^\top P_{1}\Phi_{1} \theta_1$, where $\Phi_{1} \in \mathbb{R}^{|\mathcal{S}|}$ is the matrix of state-action features corresponding to action $1$, and $r_{1} \in \mathbb{R}^{|\mathcal{S}|}$ is the vector of rewards for action 1. We solve an analogous equation for $\theta_{0}, \lambda_{0}$. This effectively extends the algorithm of Section~\cite{sec:efficient-ope} to the setting of linear function approximation. This has no hyperparameters.
  \item Off-Policy with TD, where $Q$ -functions and off-policy average rewards are calculated according to the extension of \cite{wanLearningPlanningAverageReward2021} to linear function approximation, as provided in \cite{zhangAverageRewardOffPolicyPolicy2021}. This approach has two hyperparameters: the learning rate for the $Q$ -function $\gamma / \sqrt{t}$, and the learning rate for the mean reward estimate $\beta \gamma / \sqrt{t}$.
  \item Off-Policy with Gradient TD (GTD), as in \cite{zhangAverageRewardOffPolicyPolicy2021}. This has the same  hyperparameters  $\beta, \gamma$ as TD.
\end{enumerate}

A single hyperparameter was selected for each algorithm across all treatment effect settings, based on a scalarization of MSE across all settings, and tuned on 10 held-out trajectories for each setting.

\begin{table}[htbp]
  \centering
  \begin{tabular}{r| l }
    Algorithm & Hyperparameters \\
    \hline
    TSRI & $p = {\bf 0.9}, p_{L} = {\bf 0.6}, k \in \{ {\bf 1, 2} \}, \beta = {\bf 0.2}$  \\
    DQ (LSTD) & $\alpha \in \{ 0.01, 0.1, {\bf 1}, 10, 100\}$  \\
Off-Policy (TD) & $\beta \in \{0.2, {\bf 0.5} \}, \gamma \in \{ 0.001, {\bf 0.01}, 0.1, 1.\}$ \\
Off-Policy (GTD) & $\beta \in \{0.2, {\bf 0.5} \}, \gamma \in \{ 0.001, {\bf 0.01}, 0.1, 1.\}$
  \end{tabular}
  \caption{Hyperparameters for the ridesharing setting. Parameter settings reported in the main text are in bold.}
  \label{tbl:ridesharing-hyperparams}
\end{table}

\subsubsection{Computing environment} These experiments were performed on an internal cluster. Each run of the simulator took an average of four hours, allocating a single CPU and 8GB of RAM.

%%% Local Variables:
%%% mode: latex
%%% TeX-master: "neurips_2022.tex"
%%% End:

\end{document}